\documentclass[aps,pre,reprint,superscriptaddress,nofootinbib,amsmath,amssymb]{revtex4-2}

\usepackage{graphicx}% Include figure files
\usepackage{dcolumn}% Align table columns on decimal point
\usepackage{bm}% bold math
\usepackage{float}
\usepackage{subcaption}
\usepackage{multirow}
\usepackage{comment}
\usepackage{dsfont}
\usepackage{amsthm}
\usepackage{xcolor}

\def\ie{{\frenchspacing\it i.e.}}
\def\eg{{\frenchspacing\it e.g.}}
\def\etc{{\frenchspacing\it etc.}}

\newcommand{\mat}[1]{\mathbf{#1}}
\newcommand{\x}{\mat{x}}

\newcommand{\q}{\mat{q}}
\newcommand{\qd}{\dot{\mat{q}}}
\newcommand{\qdd}{\ddot{\mat{q}}}
\newcommand{\lag}{\mathcal{L}}

\newcommand{\rvec}{\boldsymbol{r}}
\newcommand{\fvec}{\boldsymbol{f}}
\newcommand{\vvec}{\boldsymbol{v}}
\newcommand{\w}{\boldsymbol{w}}

\makeatletter
\newtheorem*{rep@theorem}{\rep@title}
\newcommand{\newreptheorem}[2]{%
\newenvironment{rep#1}[1]{%
 \def\rep@title{#2 \ref{##1}}%
 \begin{rep@theorem}}%
 {\end{rep@theorem}}}
\makeatother

\newtheorem{theorem}{Theorem}
\newreptheorem{theorem}{Theorem}

\newtheorem{lemma}{Lemma}

\newtheorem{corollary}{Corollary}[theorem]
\newcommand{\tabincell}[2]{\begin{tabular}{@{}#1@{}}#2\end{tabular}}
\newtheorem{definition}{Definition}[section]

\def\spose#1{\hbox to 0pt{#1\hss}}
\def\simlt{\mathrel{\spose{\lower 3pt\hbox{$\mathchar"218$}}
     \raise 2.0pt\hbox{$\mathchar"13C$}}}
\def\simgt{\mathrel{\spose{\lower 3pt\hbox{$\mathchar"218$}}
     \raise 2.0pt\hbox{$\mathchar"13E$}}}
\def\simpropto{\mathrel{\spose{\lower 3pt\hbox{$\mathchar"218$}}
     \raise 2.0pt\hbox{$\propto$}}}

\def\beq#1{\begin{equation}\label{#1}}
\def\eeq{\end{equation}}
\def\beqa#1{\begin{eqnarray}\label{#1}}
\def\eeqa{\end{eqnarray}}
\def\eq#1{equation~(\ref{#1})}

\begin{document}

\preprint{APS/123-QED}

\title{Machine-Learning Non-Conservative Dynamics for New-Physics Detection}% Force line breaks with \\

\author{Ziming Liu}
\email{zmliu@mit.edu}
\affiliation{Microsoft Research Asia, Beijing, China}
\affiliation{Department of Physics, Massachusetts Institute of Technology, Cambridge, USA}
\affiliation{AI Institute for Artificial Intelligence and Fundamental Interactions (IAIFI)}
\author{Bohan Wang}
\affiliation{Microsoft Research Asia, Beijing, China}
\author{Qi Meng}
\affiliation{Microsoft Research Asia, Beijing, China}
\author{Wei Chen}
\email{wche@microsoft.com}
\affiliation{Microsoft Research Asia, Beijing, China}
\author{Max Tegmark}
\email{tegmark@mit.edu}
\affiliation{Department of Physics, Massachusetts Institute of Technology, Cambridge, USA}
\affiliation{AI Institute for Artificial Intelligence and Fundamental Interactions (IAIFI)}
\author{Tie-Yan Liu}
\email{tyliu@microsoft.com}
\affiliation{Microsoft Research Asia, Beijing, China}

\date{\today}% It is always \today, today,

\begin{abstract}
	Energy conservation is a basic physics principle, the breakdown of which often implies \textit{new physics}. This paper presents a method for data-driven ``new physics" discovery. Specifically, given a trajectory governed by unknown forces, our \textit{Neural New-Physics Detector (NNPhD)} aims to detect new physics by decomposing the force field into conservative and non-conservative components, which are represented by a Lagrangian Neural Network (LNN) and a universal approximator network (UAN), respectively, trained to minimize the force recovery error plus a constant $\lambda$ times the magnitude of the predicted non-conservative force. We show that a phase transition occurs at $\lambda=1$, universally for arbitrary forces. We demonstrate that NNPhD successfully discovers new physics in toy numerical experiments, rediscovering friction (1493) from damped double pendulum, Neptune from Uranus' orbit (1846) and gravitational waves (2017) from an inspiraling orbit. We also show how NNPhD coupled with an integrator outperforms previous methods for predicting the future of a damped double pendulum.
\end{abstract}

\maketitle

\section{Introduction}

Energy conservation is a fundamental physical law, so when non-conservation is observed, physicists often consider it evidence of an unseen body or novel external forces rather than questioning the conservation law itself. In this paper, we will therefore refer to energy non-conservation as simply 
\textit{new physics} and strive to auto-detect it\footnote{In contrast, ``new physics" in high energy physics specifically refers to ``new fundamental particles" or ``new fundamental interactions" beyond the Standard Model~\cite{burgess2007standard}.}.
Many experimental new physics discoveries have manifested as apparent violation of energy conservation, for example friction \cite{HUTCHINGS201651}, Neptune \cite{Leverrier}, neutrinos \cite{Cowan1956},
dark matter \cite{Zwicky,Rubin},
extra-solar planets \cite{wolszczan1992planetary} and gravitational waves \cite{HulseTaylor}. We focus on classical mechanics in this paper, but the idea extends to all fields of physics including quantum mechanics. We illustrate several classic examples in FIG.~\ref{fig:nnphd_a}. 
In these cases, the new physics was historically identified from the residual force after fitting data to a conservative force of a {\it known functional form}. 
The key novel contribution in this paper is that our proposed model, dubbed the Neural New Physics Detector (NNPhD), can discover the new physics even when the form of the conservative ``old physics" is {\it not known}.

Data-driven discovery has proven extremely useful in physics, yet also non-trivial. For example, Kepler spent 25 years analyzing astronomical data before formulating his eponymous three laws. In this paper, we aim to automate and accelerate data-driven new physics discovery using machine learning tools. More concretely, given the trajectory of one or several objects governed by some force, we aim to decompose the force into conservative and non-conservative parts, followed by a symbolic regression module for explanation. As a trivial example, we aim to decompose the force $f=-kq-\gamma\dot{q}$ of a damped harmonic oscillator into conservative part $f_{\mathrm{c}}=-kq$ and a non-conservative part $f_{\mathrm{n}}=-\gamma\dot{q}$. 

\onecolumngrid

\begin{figure*}[h]
	\centering
	\includegraphics[width=0.8\linewidth,trim=0cm 1cm 0cm 0cm]{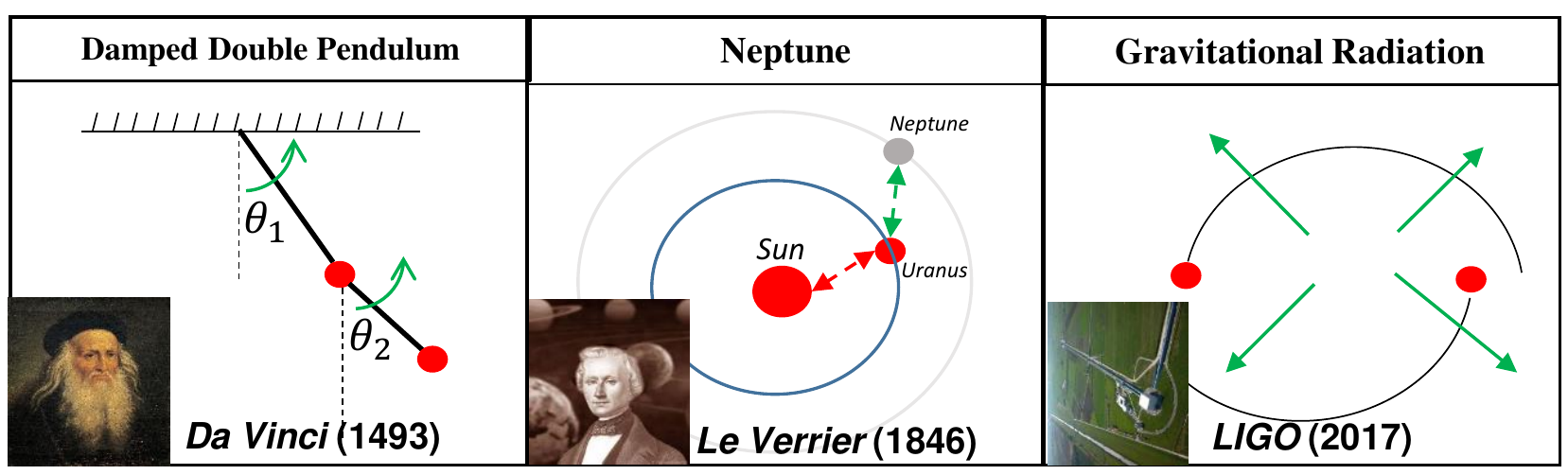}
	\vskip 0.8cm
	\caption{NNPhD can auto-rediscover several classic examples.}
	\label{fig:nnphd_a}
	\vskip -0.5cm
\end{figure*}

\twocolumngrid

\begin{figure*}[htbp]
	\centering
	\includegraphics[width=0.9\linewidth,trim=0cm 1cm 0cm 0cm]{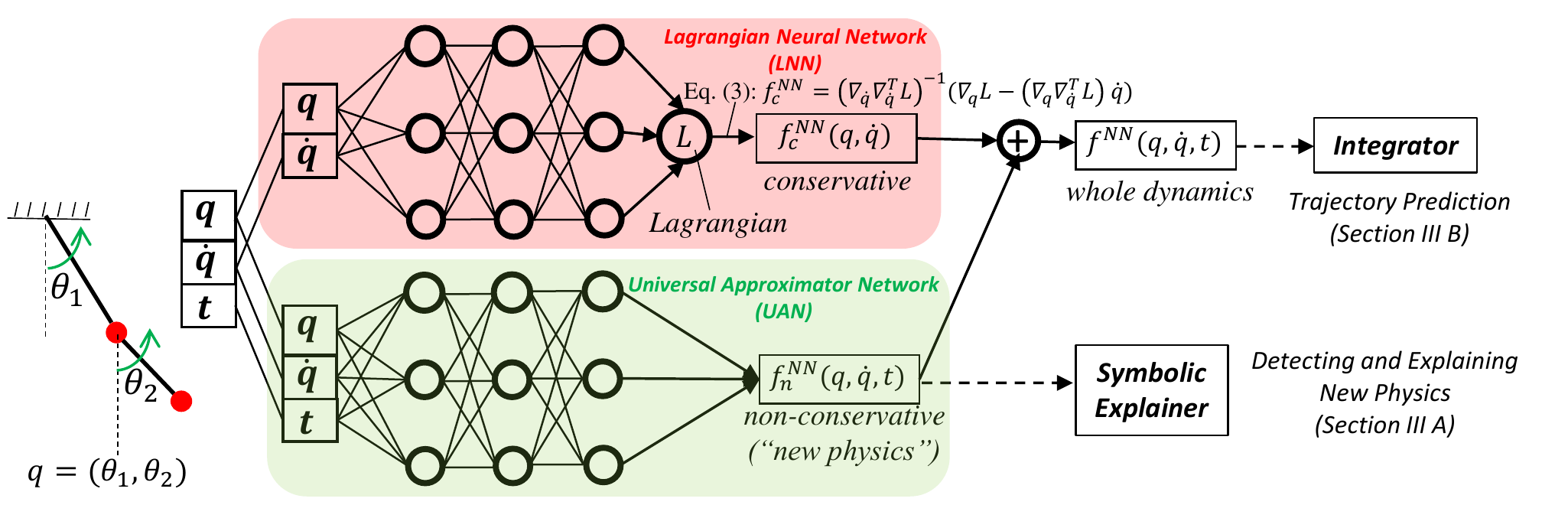}
	\vskip 0.5cm
	\caption{NNPhD predicts dynamics by decomposing the force into conservative and non-conservative components, which can reveal new physics and improve trajectory extrapolation.}
	\label{fig:nnphd_b}
	\vskip -0.3cm
\end{figure*}

Conservation laws have been introduced into neural networks as strong inductive biases, such as in the Lagrangian Neural Network (LNN)~\cite{cranmer2020lagrangian}, the Hamiltonian Neural Network (HNN)~\cite{greydanus2019hamiltonian} and variants~\cite{finzi2020simplifying, choudhary2020forecasting,jin2020sympnets, HGN}. The limitation of these models lies in their inability to model non-conservative dynamics, where the non-conservation can be caused by dissipation, external driving forces, etc. Our proposed NNPhD can resolve this limitation by augmenting LNN with a universal approximation network (UAN), illustrated in FIG.~\ref{fig:nnphd_b}. Although prior works~\cite{long2018pde,chen2018neural,battaglia2016interaction,NEURIPS2019_b2945042,PhysRevX.10.031056,champion2019data,udrescu2020symbolic,kim2019integration,raissi2019physics,cranmer2020lagrangian,greydanus2019hamiltonian,matsubara2019deep,cranmer2020discovering,Raissi1026,PhysRevE.100.033311,PhysRevX.10.021020,lutter2019deep, cranmer2020lagrangian, finzi2020simplifying, greydanus2019hamiltonian, choudhary2020forecasting,jin2020sympnets, HGN, poincare,welch1995introduction, guen2020augmenting,ajay2018augmenting,wu2015galileo,zhong2020dissipative} attempt to learn the general or conservative force from data, most of these methods are unable to perform force decomposition, except for ~\cite{ guen2020augmenting,zhong2020dissipative} which assume (partial) knowledge of physics thus lose generality. Moreover, while the current literature mostly focuses on model \textit{predictability}, we pay extra attention to \textit{explainability} made possible by symbolic regression.

The rest of this paper is organized as follows: In Section \ref{MethodSec}, we review the problem framing and useful results, define {\it force decomposition with minimal non-conservation} and propose NNPhD to learn this force decomposition. In Section \ref{sec:experiments}, we carry out numerical experiments to verify our theoretical analysis of the presented algorithm, as well as to demonstrate the potential of NNPhD for new-physics discovery.

\section{Method}
\label{MethodSec}

\subsection{Notation}

We consider the general classical physical system described by an $n$-dimensional vector $\mat{q}$ of generalized coordinates whose time-evolution $\mat{q}(t)$ is governed by 
a second-order ordinary differential equation
\begin{equation}\label{eq:qdd=f}
	\ddot{\mat{q}}=f(\mat{q},\dot{\mat{q}},t),
\end{equation}
where $f:\mathbb{R}^{2n+1}\to\mathbb{R}^n$.
The \textit{acceleration} $\qdd$ is intimately related to \textit{force} according to Newton's second law. In the following, we for simplicity refer to $f(\mat{q},\dot{\mat{q}},t)$ as a \textit{force field} (dynamics perspective) or \textit{acceleration field} (kinematics perspective) interchangeably.
The dynamical systems in our numerical examples consist of $k$ particles in $d$ dimensions, so $n=kd$ and $\mat{q}\equiv[\mat{q}_1,\cdots,\mat{q}_k]\in\mathbb{R}^{n}$, but our NNPhD method is fully general and makes no such assumptions.

An important subset of dynamical systems are known as {\it conservative} because they conserve energy, which can be described by Euler-Lagrange Equation:
\beq{EulerLagrangeEq}
\frac{d}{dt}\nabla_{\qd} \lag=\nabla_{\q}\lag,
\eeq
where $\lag(\q,\qd)$ is the Lagrangian and $\nabla$ is the gradient operator. As reviewed in Appendix A and~\cite{cranmer2020lagrangian}, the Lagrangian mechanics formalism implies that such systems allows \eq{eq:qdd=f}
to be re-expressed as 
\begin{equation}\label{eq:AL}
	\qdd = (\nabla_{\qd} \nabla_{\qd}^T\lag)^{-1}\left[\nabla_{\q}\lag-(\nabla_{\q}\nabla_{\qd}^T\lag)\qd\right]
\end{equation}
For readers whose background is primarily in machine learning rather than physics, Appendix \ref{app:ml_lag} provides a brief review of the Lagrangian mechanics formalism that we use in this paper.

\subsection{Lagrangian neural networks}

To guarantee energy conservation, inductive biases have recently been embedded into neural networks, including Lagrangian Neural Network~\cite{cranmer2020lagrangian}, Hamiltonian Neural Network~\cite{greydanus2019hamiltonian} and variants~\cite{finzi2020simplifying, choudhary2020forecasting,jin2020sympnets, HGN}.
As shown in FIG.~\ref{fig:nnphd_b}, a LNN uses a neural network to parametrize the Lagrangian $\lag(\q,\qd)$ and output $f_{\mathrm{c}}^{NN}(\q,\qd)$ through evaluating Eq.~(\ref{eq:AL}). For a given loss function defined between model output $f_{\mathrm{c}}^{NN}(\q,\qd)$ and ground truth $\qdd$, the LNN parameters can be learned  using standard  optimization algorithms. A trained LNN therefore contains a Lagrangian that determines conservative dynamics. Since not all physical systems conserve energy, the  Lagrangian mechanics is insufficient for describing non-conservative dynamics, motivating the NNPhD framework.

\subsection{The force decomposition minimizing non-conservation}\label{sec:problem_formulation}

Following the problem setting of LNN, we focus on the simple setting where the acceleration $\qdd$ is a known function i.e.,  $\qdd\equiv f(\q,\qd,t)$. Our goal is therefore not to {\it learn} the force field, but to {\it decompose} the force field. In practice, where only discrete points on trajectory $\{(\q^{(i)},t^{(i)})\}$ are known, $\qd^{(i)}$ and $\qdd^{(i)}$ can be extracted using a Neural ODE module~\cite{chen2018neural}.

The main goal is to decompose the force field $f(\q,\qd,t):\mathbb{R}^{2n+1}\to\mathbb{R}^{n}$ into a (time-independent) conservative component $f_{\mathrm{c}}(\q,\qd):\mathbb{R}^{2n}\to\mathbb{R}^{n}$ and a non-conservative component $f_{\mathrm{n}}(\q,\qd,t):\mathbb{R}^{2n+1}\to\mathbb{R}^{n}$ such that 
\beq{DecompositionEq}
f(\q,\qd,t)=f_{\mathrm{c}}(\q,\qd)+f_{\mathrm{n}}(\q,\qd,t).
\eeq
In general, the decomposition is not unique. We desire the decomposition that minimizes the non-conservative component $f_{\mathrm{n}}(\q,\qd,t)$. To define the distance between two functions, we embed all functions $f(\q,\qd,t)$ in a normed vector space $(\mathcal{F},\|\cdot\|)$ and define its conservative subspace $\mathcal{F}_{\mathrm{c}}\subset\mathcal{F}$ as
\begin{align*}
	&\mathcal{F}_{\mathrm{c}}=\left\{f\in\mathcal{F}\big|^{\ \exists \mathcal{L}(\q,\qd):\mathbb{R}^{2n}\rightarrow\mathbb{R}, \quad\textit{s.t.}}_{\ f(\q,\qd)=(\nabla_{\qd} \nabla_{\qd}^T\lag)^{-1}\left(\nabla_{\q}\lag-(\nabla_{\q}\nabla_{\qd}^T\lag)\qd\right)}\right\}.
\end{align*}
We formally define the force decomposition as follows:
\begin{definition} \textbf{(force decomposition with minimal non-conservation)} 
	{\normalfont The} conservative component {\normalfont of $f(\q,\qd,t)$ is defined as}
	\begin{align}\label{eq:non-conservation-minimization}
		f_{\mathrm{c}}(\q,\qd)\equiv\arg\min_{g\in\mathcal{F}_{\mathrm{c}}}\|f(\q,\qd,t)-g(\q,\qd)\|.
	\end{align} {\normalfont We denote $f_{\mathrm{n}}(\q,\qd,t)\equiv f(\q,\qd,t)-f_{\mathrm{c}}(\q,\qd)$ the} non-conservative component {\normalfont of $f$ and denote the decomposition $f(\q,\qd)=f_{\mathrm{c}}(\q,\qd)+f_{\mathrm{n}}(\q,\qd,t)$ the \textit{force decomposition minimizing non-conservation}.}
\end{definition}

\subsection{Neural New-Physics Detector (NNPhD) framework}\label{sec:nnphd}

To learn the force decomposition minimizing non-conservation, we define a learning framework dubbed the \textit{Neural New-Physics Detector} (NNPhD).
Specifically, NNPhD learns $f_{\mathrm{c}}$ and $f_{\mathrm{n}}$ jointly. 
As illustrated in FIG.~\ref{fig:nnphd_b}, NNPhD consists of two parallel modules, a Lagrangian Neural Network (LNN) and a Universal Approximator Network (UAN). The LNN takes in $(\q,\qd)$ to predict a Lagrangian $\lag(\q,\qd;\w_{\mathrm{c}})$ in the intermediate layer and outputs $f_{\mathrm{c}}^{NN}(\q,\qd;\w_{\mathrm{c}})$ calculated from Eq.~(\ref{eq:AL}), where $\w_{\mathrm{c}}$ are LNN parameters. The UAN is a pure black box (a fully connected neural network) that takes in $(\q,\qd,t)$ and outputs $f_{\mathrm{n}}^{NN}(\q,\qd,t;\w_{\mathrm{n}})$ where $\w_{\mathrm{n}}$ are parameters of the black box. The two outputs are summed to predict the full force field
\begin{equation}
    f^{NN}(\q,\qd,t;\w_{\mathrm{c}},\w_{\mathrm{n}})=f_{\mathrm{c}}^{NN}(\q,\qd;\w_{\mathrm{c}})+f_{\mathrm{n}}^{NN}(\q,\qd,t;\w_{\mathrm{n}}).
\end{equation}

{We take both \textit{recovery error} and \textit{minimal non-conservation} into considerations to design our loss function:} (1) $f^{NN}$ should recover ground truth $f$; (2) we make maximal use of $f_{\mathrm{c}}^{NN}$ and reduce $f_{\mathrm{n}}^{NN}$ as much as possible (e.g. when $f$ is conservative, we hope that  $f_{\mathrm{n}}^{NN}$ vanishes). Guided by these two principles, we define our loss function as follows (denoting the $i^{th}$ sample $\x^{(i)}\equiv (\q^{(i)},\qd^{(i)})$):
{\small\begin{equation}\label{eq:loss}	    
\begin{aligned}
	&L_{NNPhD}(\w_{\mathrm{c}},\w_{\mathrm{n}})= L_e(\w_{\mathrm{c}},\w_{\mathrm{n}}) + \lambda L_b(\w_{\mathrm{n}}),\\
    &L_b(\w_{\mathrm{n}}) \equiv \left(\frac{1}{Nn}\sum_{i=1}^N ||f_{\mathrm{n}}^{NN}(\x^{(i)},t^{(i)};\w_{\mathrm{n}})||^p\right)^{\frac{1}{p}},\\
        &L_e(\w_{\mathrm{c}},\w_{\mathrm{n}}) \equiv \\ &\left(\frac{1}{Nn}\sum_{i=1}^N ||f_{\mathrm{c}}^{NN}(\x^{(i)};\w_{\mathrm{c}})+f_{\mathrm{n}}^{NN}(\x^{(i)},t^{(i)};\w_{\mathrm{n}})-f(\x^{(i)},t^{(i)})||^p\right)^{\frac{1}{p}},
\end{aligned}
\end{equation}}

{\noindent where $p\geq 1$ and the regularization coefficient $\lambda>0$. The factors $\frac{1}{N}$ and $\frac{1}{n}$ average over samples and degrees of freedom, respectively.
Here we use $L_p$ function norms, i.e. $\|f\|\equiv (\int |f|^pd\mu)^{1/p}$, where the integral is replaced by averaging over finite training samples. $L_e$ is the recovery error and $L_b$ penalizes the black box module to discourage it from learning conservative dynamics.}

\subsection{The regularization phase transition}\label{subsec:phase_transition}

Does minimizing Eq.~(\ref{eq:loss}) yield the force decomposition of Eq.~(\ref{eq:non-conservation-minimization})? We offer an affirmative answer to this question by presenting Theorem \ref{thm:1} informally here. Appendix~\ref{sec:Theorem1_formal} provides a rigorous formulation and proof of this theorem.

\begin{reptheorem}{thm:1}
(Informal) Suppose $f_{\mathrm{c}}^{NN}$ and $f_{\mathrm{n}}^{NN}$ can represent any conservative force field and any (continuous) force field, and $(f_{\mathrm{c}}^*,f_{\mathrm{n}}^*)$ denotes the pair that minimizes NNPhD loss from Eq.~(\ref{eq:loss}). Then we have a phase transition at $\lambda=1$ such that $(f_{\mathrm{c}}^*,f_{\mathrm{n}}^*)=(f_{\mathrm{c}},f_{\mathrm{n}})$ when $0<\lambda<1$, and $(f_{\mathrm{c}}^*,f_{\mathrm{n}}^*)=(f_{\mathrm{c}},0)$ when $\lambda>1$.
\end{reptheorem}

Theorem \ref{thm:1} has two interesting and useful implications: (1) {\it sharp phase transition}: 
The recovery error $L_e=0$ when $\lambda<1$ and $L_e=\|f_{\mathrm{n}}\|>0$ when $\lambda>1$. As a result, non-conservative dynamics predicts an error jump of $L_e$ at $\lambda=1$, while conservative dynamics does not. This phenomenon justifies the term ``detector" in our model name, in the sense that non-conservative dynamics is \textit{detected} by the sharp phase transition at $\lambda=1$. (2) {\it effortless $\lambda$ tuning}: Any $\lambda\in(0,1)$ would achieve the force decomposition. Below we report numerical experiments showing that in practice, too small $\lambda$ do not regularize UAN effectively, and force decomposition results are robust for $0.05\lesssim\lambda<1$ independent of dynamical systems at study.

As we will see in Appendix~\ref{sec:Theorem1_formal}, the proof is more complicated than one might naively expect. If conservative force fields formed a linear subspace, then 
the conservative component $f_c$ from \eq{eq:non-conservation-minimization} 
would simply be the orthogonal projection onto that space, and the non-conservative residual $f_n$ would be orthogonal to that subspace.
But conservative force fields as we have defined them generally do {\it not} form a linear subspace, \ie, the sum of two energy-conserving force fields may not conserve energy, which is related to the nonlinear nature of \eq{eq:AL}.

\section{Results from numerical experiments}\label{sec:experiments}

In this section, we test our NNPhD algorithm with a series of numerical examples defined in Table I. In Section \ref{sec:exp_new_physics}, we quantify its ability to rediscover symbolic expressions for ``new physics" such as friction, Neptune and gravitational waves. %and unseen ``dark" particles.
In Section \ref{sec:prediction_exp}, we show that, although NNPhD is designed for new physics detection, it can also outperform baseline trajectory prediction for the damped double pendulum example.
In Section \ref{sec:exp_decomp_subsec}, we use toy examples to verify and quantify the aforementioned $\lambda$-dependent phase transition, and explore how the choices of $p$ and $\lambda$ in Eq.~(\ref{eq:loss}) influence algorithm behavior. Finally we discuss how data quality affects identifiability of new physics in Section~\ref{sec:data_quality_exp}. Further technical details on model parameters, simulations and neural network architecture are provided in Appendix \ref{app:equations}.

\subsection{Discovery of New Physics}\label{sec:exp_new_physics}

We now test NNPhD on three numerical examples defined in Table \ref{tab:four_new_physics_equations}, to see if it can rediscover friction (1493),  Neptune (1846) and gravitational wave emission (2017). %and toy ``dark matter" \cite{HUTCHINGS201651,neptune,abramovici1992ligo,bernabei2003dark}, using the dynamical equations given in Table \ref{tab:four_new_physics_equations}.
In all three cases, the force fields defined by the right hand side are the sum of a conservative part (the first term) and a non-conservative ``new physics" part (the second term) that we hope to discover. Before delving into our numerical experiments, let us briefly comment on how we model these three dynamical systems.

\subsubsection{Physical systems tested}\label{sec:physica_system_tested}

\begin{table*}[ht]
    \centering
    \begin{tabular}{|c|c|}\hline
    Model    &  Equation \\\hline
    \shortstack{Damped \\ Double Pendulum}  &  $ \begin{pmatrix}
		\ddot{\theta}_1\\
		\ddot{\theta}_2
	\end{pmatrix}=
	\begin{pmatrix}
		\frac{m_2l_1\dot{\theta}_1^2{\rm sin}(\theta_2-\theta_1){\rm cos}(\theta_2-\theta_1)+m_2g{\rm sin}\theta_2+m_2l_2\dot{\theta}_2^2{\rm sin}(\theta_2-\theta_1)-(m_1+m_2)g{\rm sin}\theta_1}{(m_1+m_2)l_1-m_2l_1{\rm cos}^2(\theta_2-\theta_1)}\\
		\frac{-m_2l_2\dot{\theta}_2^2{\rm sin}(\theta_2-\theta_1)+(m_1+m_2)(g{\rm sin}\theta_1{\rm cos}(\theta_2-\theta_1)-l_1\dot{\theta}_2^2{\rm sin}(\theta_2-\theta_1)-g{\rm sin}\theta_2)}{(m_1+m_2)l_1-m_2l_1{\rm cos}^2(\theta_2-\theta_1)}
	\end{pmatrix}
	-\gamma \begin{pmatrix}
		\dot{\theta}_1\\
		\dot{\theta}_2
	\end{pmatrix}$ \\\hline 
	Neptune & $\begin{pmatrix}
		\ddot{x}\\
		\ddot{y}\\
	\end{pmatrix}=
	\begin{pmatrix}
		-\frac{GM_\odot x}{(x^2+y^2)^{\frac{3}{2}}}\\
		-\frac{GM_\odot y}{(x^2+y^2)^{\frac{3}{2}}}\\
	\end{pmatrix}
	+
	\begin{pmatrix}
		\frac{ GM_n(-x+r_n{\rm cos}(\omega_n t))}{[(x-r_n{\rm cos}(\omega_n t))^2+(y-r_n{\rm sin}(\omega_n t))^2]^{\frac{3}{2}}}\\
		\frac{ GM_n(-y+r_n{\rm sin}(\omega_n t))}{[(x-r_n{\rm cos}(\omega_n t))^2+(y-r_n{\rm sin}(\omega_n t))^2]^{\frac{3}{2}}}\\
	\end{pmatrix}$\\\hline
	Gravitational Radiation & $\begin{pmatrix}
		\ddot{x}\\
		\ddot{y}\\
	\end{pmatrix}=
	\begin{pmatrix}
		-\frac{G(M_1+M_2)x}{(x^2+y^2)^{\frac{3}{2}}}\\
		-\frac{G(M_1+M_2)y}{(x^2+y^2)^{\frac{3}{2}}}\\
	\end{pmatrix}
	+\frac{32M_1M_2(M_1^2+M_2^2)}{5Gc^5(M_1+M_2)^5}
	\begin{pmatrix}
		-(\dot{x}_i^2+\dot{y}_i^2)^{4}\dot{x}_i\\
		-(\dot{x}_i^2+\dot{y}_i^2)^{4}\dot{y}_i\\
	\end{pmatrix}
	$ \\\hline
	%``Dark Matter" & $\begin{pmatrix}
	%	\ddot{x}_i\\
	%	\ddot{y}_i\\
	%\end{pmatrix}=
	%\sum_{j=1}^{N_{o}}
	%\begin{pmatrix}
	%	-\frac{\partial U(r_{ij})}{\partial x_j}\\
	%	-\frac{\partial U(r_{ij})}{\partial y_j}\\
	%\end{pmatrix}
	%+
	%\sum_{k=N_o+1}^{N}
	%\begin{pmatrix}
	%-\frac{\partial U(r_{ik})}{\partial x_k}\\
	%	-\frac{\partial U(r_{ik})}{\partial y_k}\\
	%\end{pmatrix},\quad i=1,\dots, N_o$
	%\\\hline
    \end{tabular}
    \caption{We test if NNPhD can automatically decompose these three force fields into a conservative part (first term) and a non-conservative part (second term) corresponding to the ``new physics".}
    \label{tab:four_new_physics_equations}
\end{table*}

{\bf Friction} Italian polymath Leonardo da Vinci first recorded the basic laws of friction in 1493. We add friction to the double pendulum system and to test if NNPhD can automatically discover the friction force solely from data. The damped double pendulum example can be described by two angles and their derivatives \ie, 
$\q=(\theta_1,\theta_2)$ and $\qd=(\dot{\theta}_1,\dot{\theta}_2)$.
In our numerical experiment, we choose the physical parameters $m_1=m_2=g=l_1=l_2=1$, 
$\gamma=0.02$.

{\bf Neptune} \textit{Le Verrier} postulated the existence of Neptune in 1846: astronomers had found that Uranus' orbit around the Sun precessed in a way suggesting the presence of a force of unknown cause, later identified as Neptune.
Neptune was invisible at the time in the sense that contemporary astronomers could not observe its position or velocity, but Le Verrier (and NNPhD) were able to identify the existence of a third body by identifying a non-conservative contribution to the force field of the two-body system. 
For our numerical experiments, we make the simplifying assumptions that (1) the Sun remains fixed at the origin, (2) the elliptical orbits of Uranus and Neptune are circular (have  eccentricity $e=0$) and lie in the same plane, (3) Neptune's orbit is unaffected by Uranus, and (4) the effects of other planets are negligible. Here $x$ and $y$ denote the coordinates of Uranus, and time $t$ is measured in units such that Uranus' orbital period is $2\pi\sqrt{2^3}$. We choose $G=1$, mass of Sun $M_{\odot}=1$. Neptune's mass, orbital radius and angular velocity are $M_n=0.005$, $r_n=3$ and $\omega_n=3^{-\frac{3}{2}}\approx 0.192$.

{\bf Gravitational Radiation} 
As predicted by Einstein, the gravitational two-body problem is non-conservative, since the system radiates gravitational radiation that carries away energy and causes orbital decay. Experimental confirmation of this garnered Nobel Prizes both in 1993 (for the Hulse-Taylor pulsar) and in 2017 (for the LIGO discovery of gravitational waveforms from black hole mergers), and there is great current interest in exploiting such signals both for gravitational wave astronomy and for precision tests of general relativity.
To test whether NNPhD can auto-discover the non-conservative force caused by gravitational wave back-reaction solely from black hole trajectories, we simulate a binary black hole inspiral
using the approximation from \cite{ligo2017basic} that the 
radiated gravitational wave power 
$P
=\frac{32}{5}\frac{G}{c^5}\mu^2r^4\Omega^6%\propto -r^4 \Omega^6
=\frac{32}{5}\frac{G}{c^5}\mu^2\frac{v^6}{r^2}$ in a slowly decaying circular orbit (of radius $r$, angular frequency $\Omega$ and reduced mass $\mu\equiv (M_1^{-1}+M_2^{-1})^{-1}$)
equals the energy loss rate $-dE/dt=v f$ %$=\Omega r f$ 
from a dissipative back-reaction force $f$.
Using $\Omega\propto r^{-3/2}$ and $v\propto r^{-1/2}$ from Kepler's 3rd law gives  $P\propto v^{10}$, with a total force
\beq{InspiralEq}
\fvec = \mu\ddot{\rvec} = -{GM_1M_2\over r^3}\rvec -\frac{32M_1^2M_2^2(M_1^2+M_2^2)}{5Gc^5(M_1+M_2)^6} v^8\vvec,
\eeq
corresponding to an acceleration

\beq{InspiralEq_acc}
\ddot{\rvec} = -{G(M_1+M_2)\over r^3}\rvec -\frac{32M_1M_2(M_1^2+M_2^2)}{5Gc^5(M_1+M_2)^5} v^8\vvec,
\eeq
where $\rvec=\rvec_2-\rvec_1$ and $\vvec=\vvec_2-\vvec_1$. We choose these physical parameters to be $G=M_1=M_2=1, c=3$.

\subsubsection{Detection of new physics}\label{sec:detection_subsec}

These three physical systems have $d=2$ degrees of freedom, obeying the second-order coupled differential equations 
in Table \ref{tab:four_new_physics_equations}.
Including the corresponding conjugate momenta, a system's state is thus a point moving along some
trajectory in a $2d$-dimensional phase space, satisfying a $2d$ coupled first-order coupled differential equations. We solve these equations and compute the trajectories numerically using 
a 4th-order Runge-Kutta integrator at $N_{step}=\{300,1000,300%,60
\}$ timesteps of size $\varepsilon=\{0.1,0.1,0.05%,0.01
\}$ for the three physical systems, using the following initial conditions:

\begin{eqnarray}
\begin{aligned}
    (\theta_1,\theta_2,\dot{\theta}_1, \dot{\theta}_2)&=(1,0,0,0)\\
    (x,y,\dot{x},\dot{y})&=(3,0,0,\frac{1}{\sqrt{3}})\\
    (x,y,\dot{x},\dot{y})&=(0,2,-1,0)
    %(x_i,y_i) &= (-7.5+j,-7.5+k)+(N(0,0.2^2),N(0,0.2^2))\\
    %&\quad\quad\quad\quad j=\lfloor \frac{i}{16}\rfloor, k=i-16j,
    %i=1,2,\cdots,256
\end{aligned}
\end{eqnarray}
Once trajectory points are calculated, the ground truth forces $f$ at those points are evaluated using the formula in TABLE \ref{tab:four_new_physics_equations}~\footnote{In more realistic settings, one would first extract $\qdd$ from trajectory data, e.g. with Neural ODE~\cite{chen2018neural} or AI Physicist~\cite{PhysRevE.100.033311}, and then use $\qdd$ as labels to train NNPhD. We treat $\qdd=f(\q,\qd,t)$ as an oracle in this paper since we focus on the force field decomposition aspect.}. We do not hold back any testing data in this section, since many insights can be gained solely from training data. We will hold back testing data and verify NNPhD's generalization ability in Section \ref{sec:prediction_exp}.

We then train NNPhD on the aforementioned trajectory data  as detailed in Appendix~\ref{app:training_details}.
FIG.~\ref{fig:new_physics_models}  shows the resulting NNPhD prediction loss $L_e$ as a function of $\lambda$, revealing a striking phase transition at $\lambda=1$: for $\lambda<1$, $L_e$ is almost zero, while for $\lambda>1$, $L_e$ is an approximately constant positive number, indicating the magnitude of non-conservative components.%\footnote{Note that we have divided the ``dark matter" prediction loss by 5.4 to make its magnitude comparable to the other curves; this is because the new physics is only a perturbation around old physics in the first three examples, while new physics is dominant over old physics for the dark matter example (with 5.4 times more ``dark" particles)}
 
 As we showed above, such a phase transition is a smoking-gun signature of new physics  manifesting
as non-conservative dynamics. The observed phase transitions thus justify the NNPhD name.

\begin{figure}[htbp]
	\centering
	\includegraphics[width=1.0\columnwidth, trim=0cm 1cm 0cm 0cm]{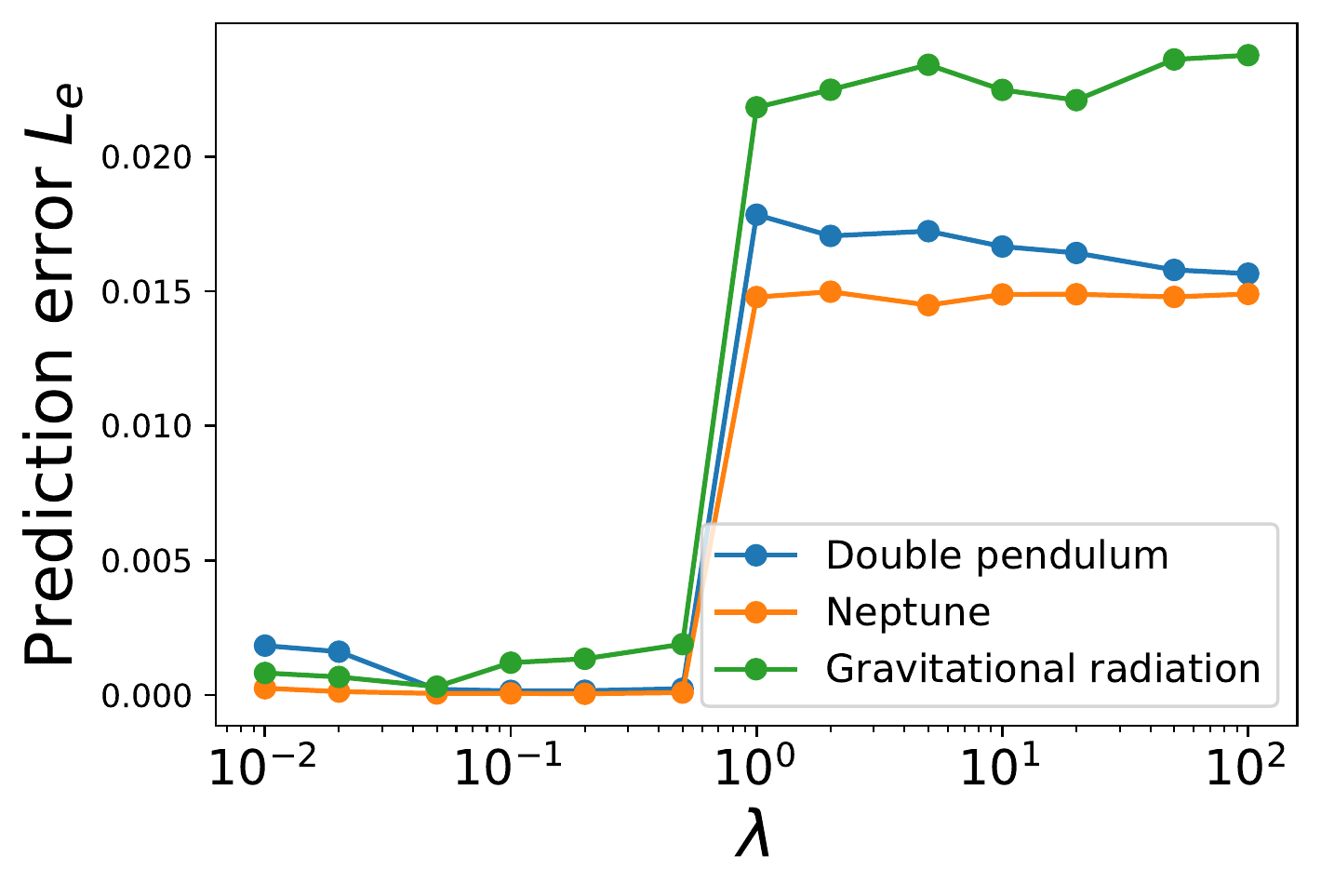}
	\caption{In all our three examples, clear phase transitions at $\lambda=1$ indicate the existence of new physics.}
	\label{fig:new_physics_models}
	\vskip -0.4cm
\end{figure}

\begin{table*}[]
	\centering
	\caption{Symbolic Formulas Discovered by NNPhD}
	\begin{tabular}{|c|c|c|c|}\hline
		Physics Example & Target & Ground Truth ``New Physics" & NNPhD+Symbolic \\\hline
		Double Pendulum & $\begin{pmatrix}
			\ddot{\theta}_1\\
			\ddot{\theta}_2\\
		\end{pmatrix}$ & 
		$\begin{pmatrix}
			-{\bf 0.02}\dot{\theta}_1-{\bf 0.00}\dot{\theta}_2\\
			-{\bf 0.00}\dot{\theta}_1-{\bf 0.02}\dot{\theta}_2\\
		\end{pmatrix}$  & 
		$\begin{pmatrix}
			-{\bf 0.018}\dot{\theta}_1-{\bf 0.001}\dot{\theta}_2\\
			-{\bf 0.001}\dot{\theta}_1-{\bf 0.018}\dot{\theta}_2
		\end{pmatrix}$
		\\\hline
		Neptune & 
		$\begin{pmatrix}
			\ddot{x}\\
			\ddot{y}\\
		\end{pmatrix}$ & 
		\scalebox{0.8}{
			$\begin{pmatrix}
				\frac{{\bf 0.005}(-x+{\bf 3}{\rm cos}({\bf 0.192}t))}{[(x-{\bf 3}{\rm cos}({\bf 0.192}t))^2+(y-{\bf 3}{\rm sin}({\bf 0.192}t))^2]^{\frac{3}{2}}}\\
				\frac{{\bf 0.005}(-y+3{\rm sin}({\bf 0.192}t))}{[(y-{\bf 3}{\rm sin}({\bf 0.192}t))^2+(y-{\bf 3}{\rm sin}({\bf 0.192}t))^2]^{\frac{3}{2}}}\\
			\end{pmatrix}$}
		& 
		\scalebox{0.8}{
			$\begin{pmatrix}
				\frac{{\bf 0.0052}(-x+{\bf 3.004}{\rm cos}({\bf 0.192}t))}{[(x-{\bf 3.004}{\rm cos}({\bf 0.192}t))^2+(y-{\bf 3.004}{\rm sin}({\bf 0.192}t))^2]^{\frac{3}{2}}}\\
				\frac{{\bf 0.0052}(-y+{\bf 3.004}{\rm sin}({\bf 0.192}t))}{[(y-{\bf 3.004}{\rm sin}({\bf 0.192}t))^2+(y-{\bf 3.004}{\rm sin}({\bf 0.192}t))^2]^{\frac{3}{2}}}\\
			\end{pmatrix}$}\\\hline
		Gravitational Radiation &
		$\begin{pmatrix}
			\ddot{x}_1\\
			\ddot{y}_1\\
		\end{pmatrix}$ & 
		$\begin{pmatrix}
			-{\bf 0.00165}(\dot{x}_1^2+\dot{y}_1^2)^{{\bf 4}}\dot{x}_1\\
			-{\bf 0.00165}(\dot{x}_1^2+\dot{y}_1^2)^{{\bf 4}}\dot{y}_1\\
		\end{pmatrix}$  & 
		$\begin{pmatrix}
			-{\bf 0.00170}(\dot{x}_1^2+\dot{y}_1^2)^{{\bf 3.94}}\dot{x}_1\\
			-{\bf 0.00170}(\dot{x}_1^2+\dot{y}_1^2)^{{\bf 3.94}}\dot{y}_1\\
		\end{pmatrix}$ \\\hline
		% Dark Matter ($U$:potential) & 
		%$\begin{pmatrix}
		%	\ddot{x}\\
		%	\ddot{y}
		%\end{pmatrix}$ &
		%$\begin{pmatrix}
		%	-\frac{\partial}{\partial x}\sum_{i=1}^{N}U(x_i,y_i,x,y)\\
		%	-\frac{\partial}{\partial y}\sum_{i=1}^{N}U(x_i,y_i,x,y)\\
		%\end{pmatrix}$  & 
		%$\begin{pmatrix}
		%	-(1+t)x+t\\
		%	-(1+t-2.98t^2)y+tx
		%\end{pmatrix}$, etc. \\\hline
	\end{tabular}
	\label{tab:symbolic_new_physics}
	\vskip -0.4cm
\end{table*}

\subsubsection{Modeling of New Physics with Symbolic Expressions}\label{sec:symbolic_subsec}

After detecting the existence of new physics, physicists 
are interested in understanding and explaining this new physics by describing it
with via symbolic expressions. 
We found that if we did not impose any inductive biases on the LNN, we
unfortunately did not auto-discover ant meaningful symbolic expressions.
We therefore drew inspiration from the history of physics, where inductive biases have routinely been used. For example, physicists often knew and used analytic formulas for the old physics when quantifying new physics.
In this spirit, we constrain the form of LNN Lagrangian so that only a set of coefficients are learnable, while the UAN remains to a fully general
feedforward neural network with two hidden layers containing 200 neurons each. Specifically, we  parametrize the Lagrangians for our three examples as follows:
{\footnotesize \begin{eqnarray}
\begin{aligned}
    &\lag_{\rm fric}=c_1{\cos}\theta_1+c_2{\rm cos}\theta_2+c_3\dot{\theta}_1^2+c_4\dot{\theta}_2^2+c_5\dot{\theta}_1\dot{\theta}_2{\rm cos}(\theta_1-\theta_2)\\
    &\lag_{\rm neptune}=c_1\dot{x}^2+c_2\dot{y}^2+\frac{c_3}{\sqrt{x^2+y^2}}\\
    &\lag_{\rm grav}=c_1\dot{x}^2+c_2\dot{y}^2+\frac{c_3}{\sqrt{x^2+y^2}}
    %&\lag_{dark}=\sum_{i=1}^{N}\sum_{j>i}^{N}cU(r_{ij}),\  r_{ij}=\sqrt{(x_i-x_j)^2+(y_i-y_j)^2}
\end{aligned}
\end{eqnarray}}
This is implemented by inputting hand-crafted features (${\rm cos}\theta, x^2$, \etc) into a learnable linear layer which outputs the predicted Lagrangian. We adopt a train-and-explain strategy: 
\begin{enumerate}
\item {\bf Training:} Like before, we train the whole NNPhD (LNN and UAN are updated simultaneously) with $\lambda=0.2$ using the ADAM optimizer with annealing learning rate $\{10^{-2}, 10^{-3}, 10^{-4}, 10^{-5}\}$ for 2000 steps. 
\item {\bf Explaining:} After training, we aim to extract more interpretable physics from the UAN via constrained nonlinear optimization of free parameters (displayed as {\bf bold} in Table \ref{tab:symbolic_new_physics}) to explain the output of the black-box, since ground truth symbolic forms are available. %For the dark matter example where no (brief) symbolic ground truth is known, we use {\tt gplearn} and {\tt AI Feynman}~\cite{udrescu2020ai, udrescu2020symbolic} for symbolic regression. 
\end{enumerate}
In Table \ref{tab:symbolic_new_physics}, we show ground truth ``new physics" and NNPhD discovered symbolic expressions. Fitted coefficients are seen to match ground truth quite well: (1) damping coefficient; (2) orbital radius and angular velocity of Neptune around the Sun; (3) magnitude and velocity dependence of gravitational wave emission.

\begin{comment}

Dark matter and ordinary matter particles are simulated as an $N$-body system interacting with only gravitational forces. Because simulation data are on the particle level, it is non-trivial to obtain a mean-field symbolic expression for gravitational field of dark matter. {\tt gplearn} uses genetic algorithms and discovers a relatively complicated but inspiring formula (see Table \ref{tab:symbolic_new_physics}). The formula is simply polynomials of input variables, yet contains important physics of gravitational collapsing and momentum conservation breakdown etc. 
On the other hand, {\tt AI Feynman} is a physics-inspired tool that discovers simpler and more elegant equations. We list a few discovered equations on the Pareto frontier~\footnote{The definition of Pareto frontier for symbolic expressions can be found in ~\cite{udrescu2020ai}.}:
\begin{equation}
	\begin{aligned}
		\ddot{x}=23t^{3/2}, -{\rm tan}({\rm sin}x)\\
		\ddot{y}=28t^{3/2}, -y{\rm exp}(-y)
	\end{aligned}
\end{equation}

The first terms can be explained as ordinary matter's breakdown of total momentum due to dark matter and the power law behavior wrt time is particularly intriguing. The second terms reveal gravitational collapsing force along $x$ and $y$ direction respectively. Although symbolic expressions do not immediately reveal hidden physics, human scientists are more likely to gain insights from symbolic expressions rather than blackboxes such as trained neural networks.

\end{comment}

\begin{figure*}[ht]
	\centering
	\begin{subfigure}[b]{0.55\textwidth}        \includegraphics[trim=0cm 0.5cm 0cm 0cm, width=\textwidth]{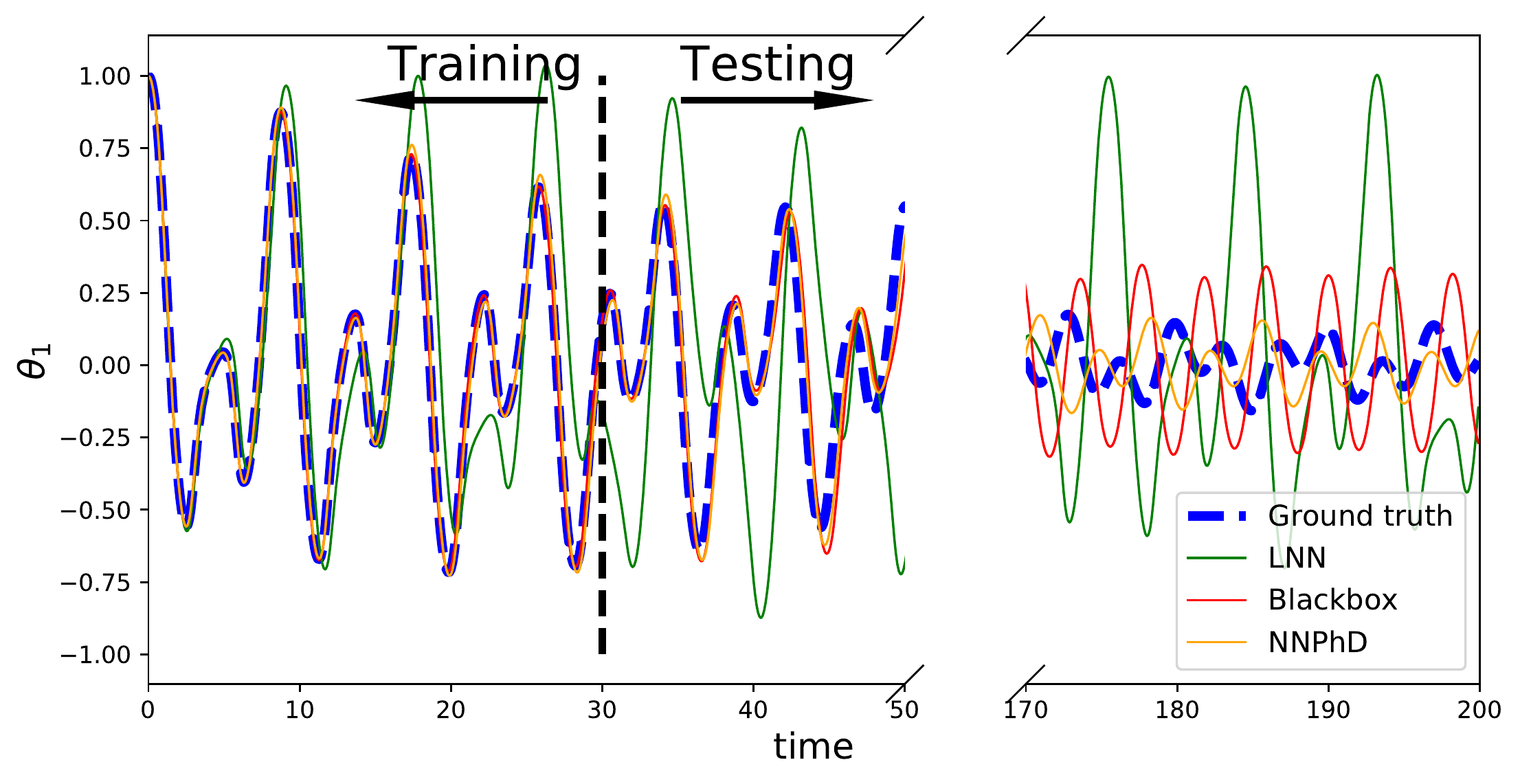}
		\caption{}
		\label{rfidtest_xaxis}
	\end{subfigure}
	\begin{subfigure}[b]{0.3\textwidth}
		\includegraphics[trim=0cm 0.5cm 0cm 0cm, width=\textwidth]{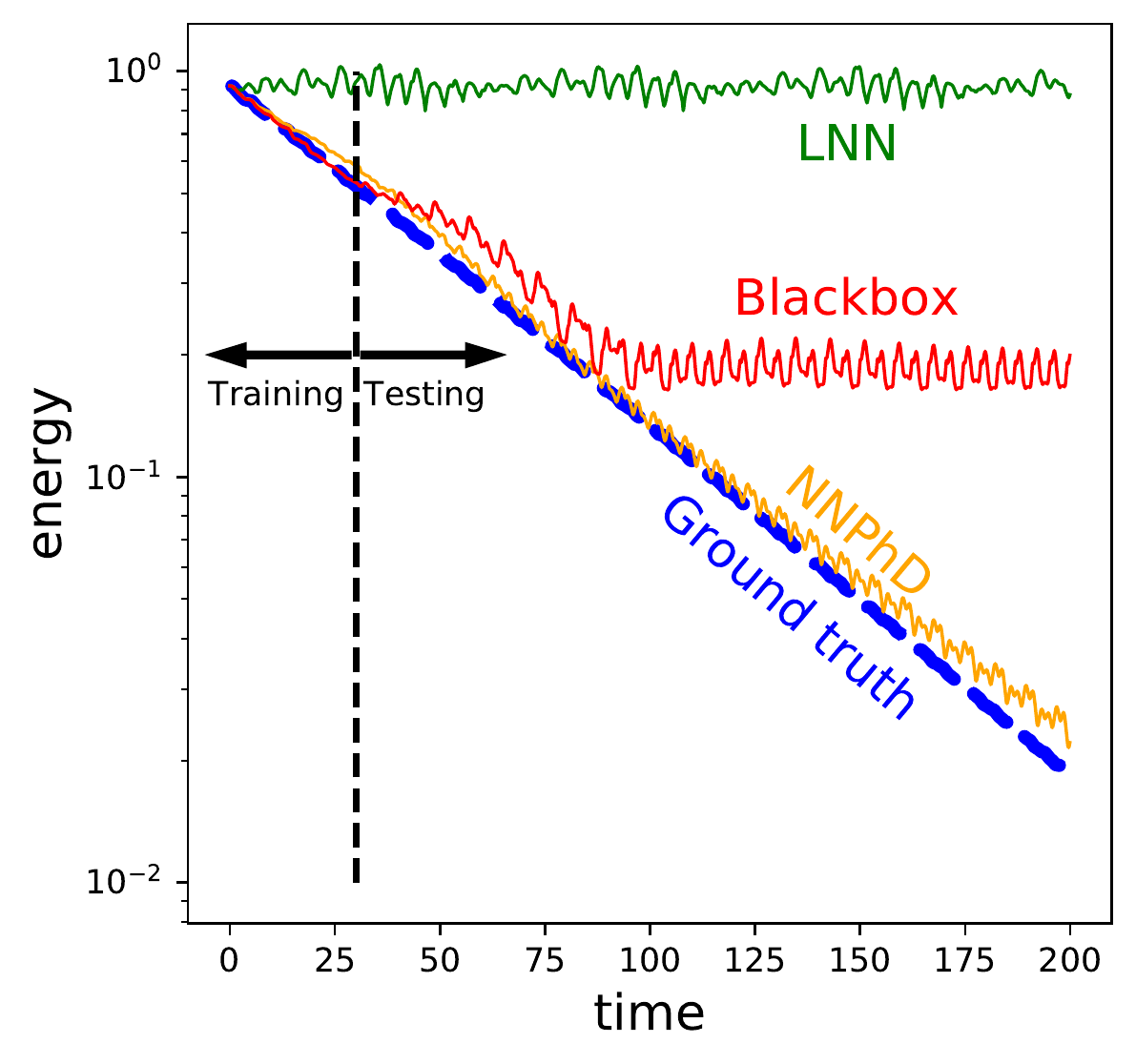}
		\caption{}
		\label{rfidtest_yaxis}
	\end{subfigure}
	\vskip -0.5cm
	\caption{Double Pendulum example: (a) Both NNPhD and the black-box  can accurately fit the angle $\theta_1$ of training samples, and can successfully extrapolate for a brief period, while LNN fails to model the non-conservative dynamics; (b) NNPhD correctly predicts the exponential energy decay on testing samples, while the black-box generalizes worse, and LNN incorrectly conserves energy.}
	\label{fig:double_pendulum}
\end{figure*}

\subsection{Prediction of Trajectories}\label{sec:prediction_exp}

In addition to discovering new physics, as we saw above, 
NNPhD can also compete with other methods on simple trajectory prediction,
and we will now test its performance for out-of-distribution generalization.
Specifically, we test how accurately it can extrapolate the trajectory
of the damped double pendulum from Section~\ref{sec:detection_subsec}, whose state is specified by two angles $(\theta_1,\theta_2)$ and corresponding angular velocities $(\dot{\theta}_1,\dot{\theta}_2)$. We compute a trajectory with a $4^{\rm th}$-order Runge-Kutta integrator at $N_{step}=2000$ timesteps of size $\varepsilon=0.1$ using the initial conditions $(\theta_1,\theta_2,\dot{\theta}_1, \dot{\theta}_2)=(1,0,0,0)$. Our test task is to extrapolate beyond $t=30$, so we split the trajectory into a training dataset $(0\leq t\leq 30)$ and a test dataset $(30\leq t\leq 200)$.

We train NNPhD with $\lambda=0.2$ and feed its prediction $f$ into a $4^{\rm th}$-order Runge-Kutta integrator to produce the predicted trajectory. 
Figure~\ref{fig:double_pendulum} compares the performance with that from
a LNN and a pure black box neural network.
The left panel shows that both NNPhD and the black box can fit $\theta_1$ well on training samples and extrapolate for a short period, but fail at larger times due to accumulated errors and sensitive phases. In contrast, we see that the LNN cannot even fit the training data,
because it has the invalid energy-conservation assumption built in.
The right panel shows that ground-truth energy is decaying exponential over time due to friction, while the LNN stubbornly predicts constant energy. 
NNPhD is seen to predict the energy decay best of the three methods, while the black-box slightly overpredicts  the 
the energy for a while and then incorrectly transitions to predicting approximate energy conservation.

\begin{figure*}[ht]
	\centering
	\begin{subfigure}[b]{0.5\columnwidth}        \includegraphics[trim=0cm 0.5cm 0cm 0cm, width=\textwidth]{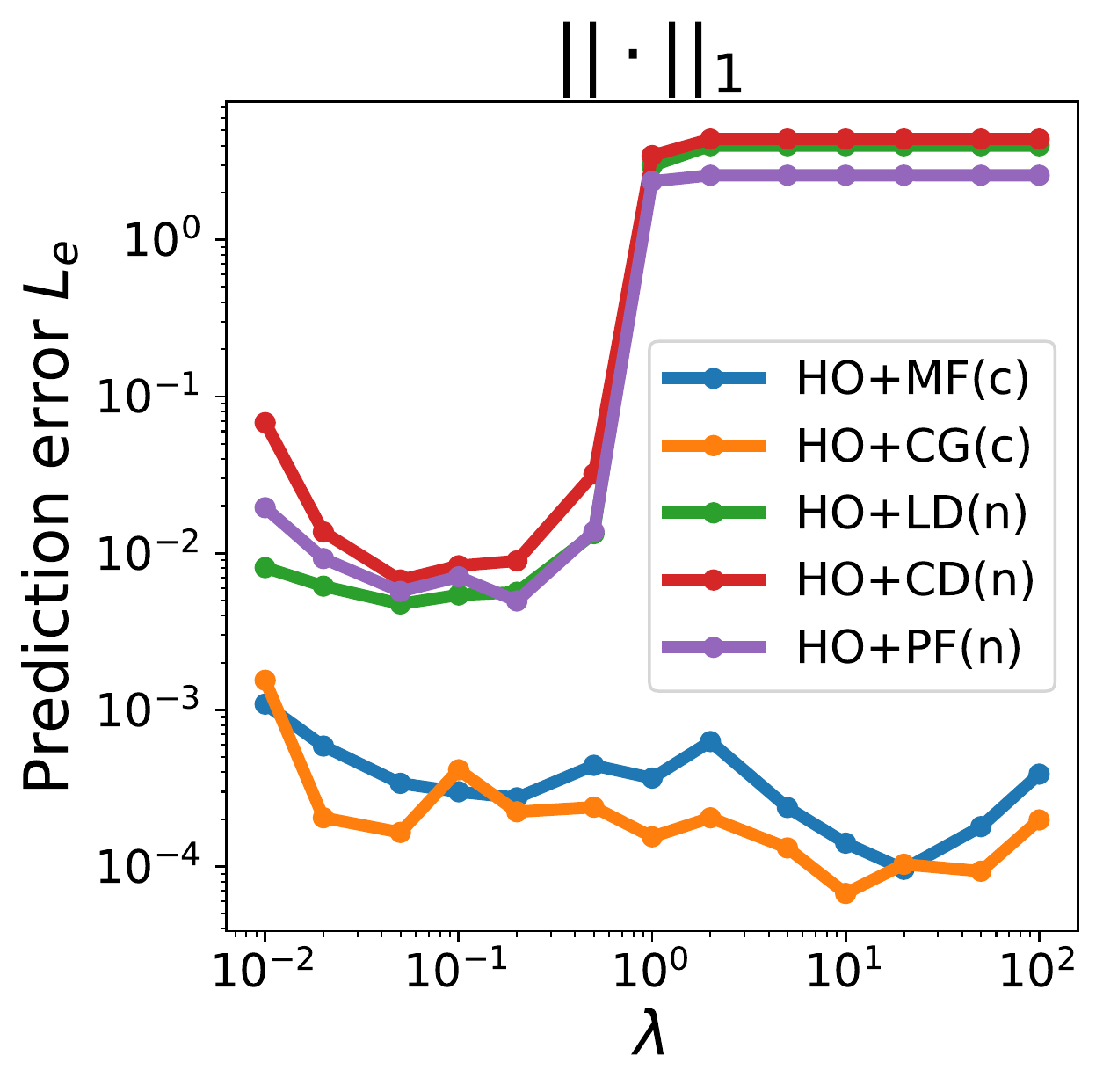}
		\caption{}
	\end{subfigure}
	\begin{subfigure}[b]{0.5\columnwidth}
		\includegraphics[trim=0cm 0.5cm 0cm 0cm,width=\textwidth]{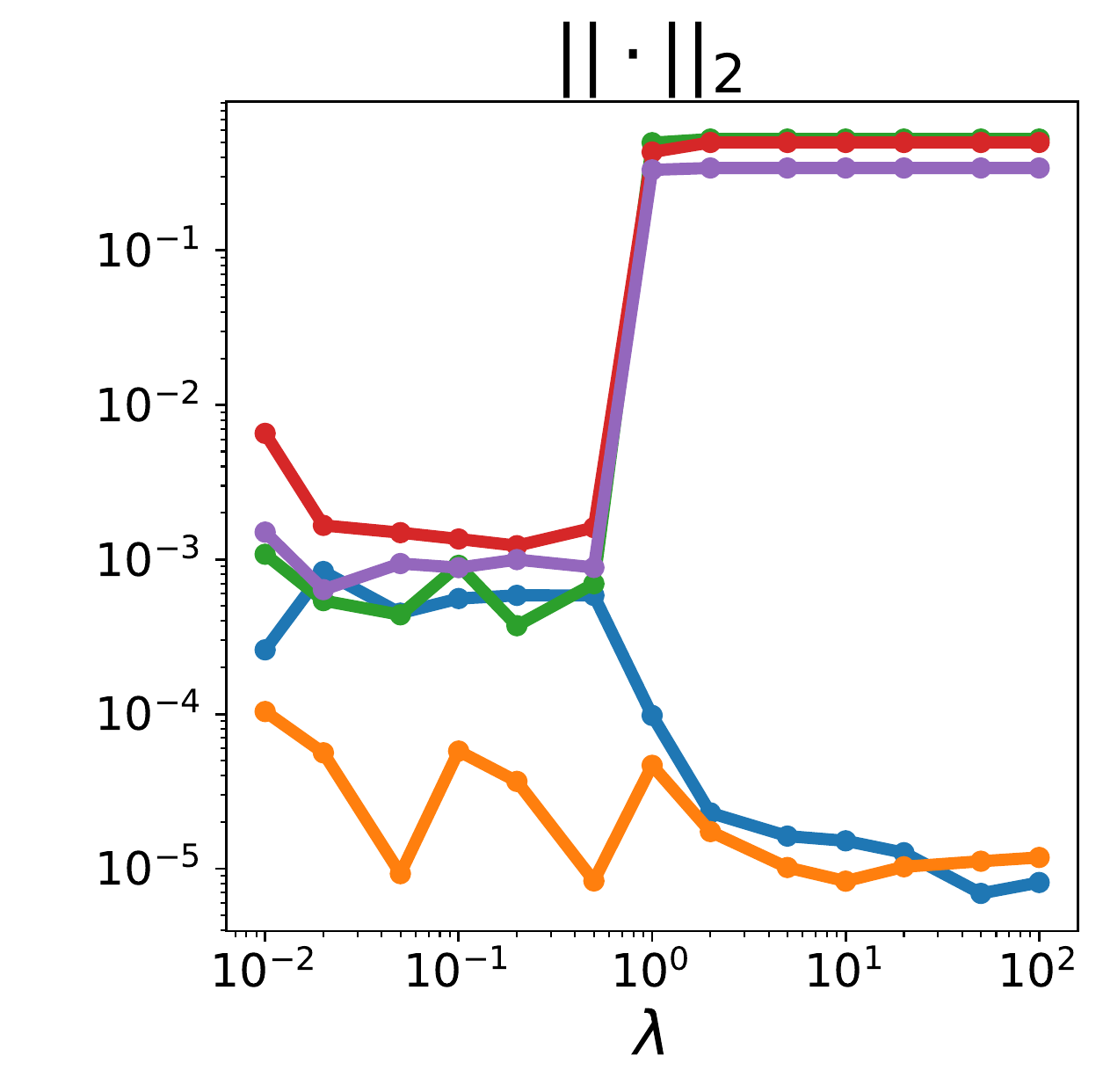}
		\caption{}
	\end{subfigure}
	\begin{subfigure}[b]{0.5\columnwidth}        \includegraphics[trim=0cm 0.5cm 0cm 0cm, width=\textwidth]{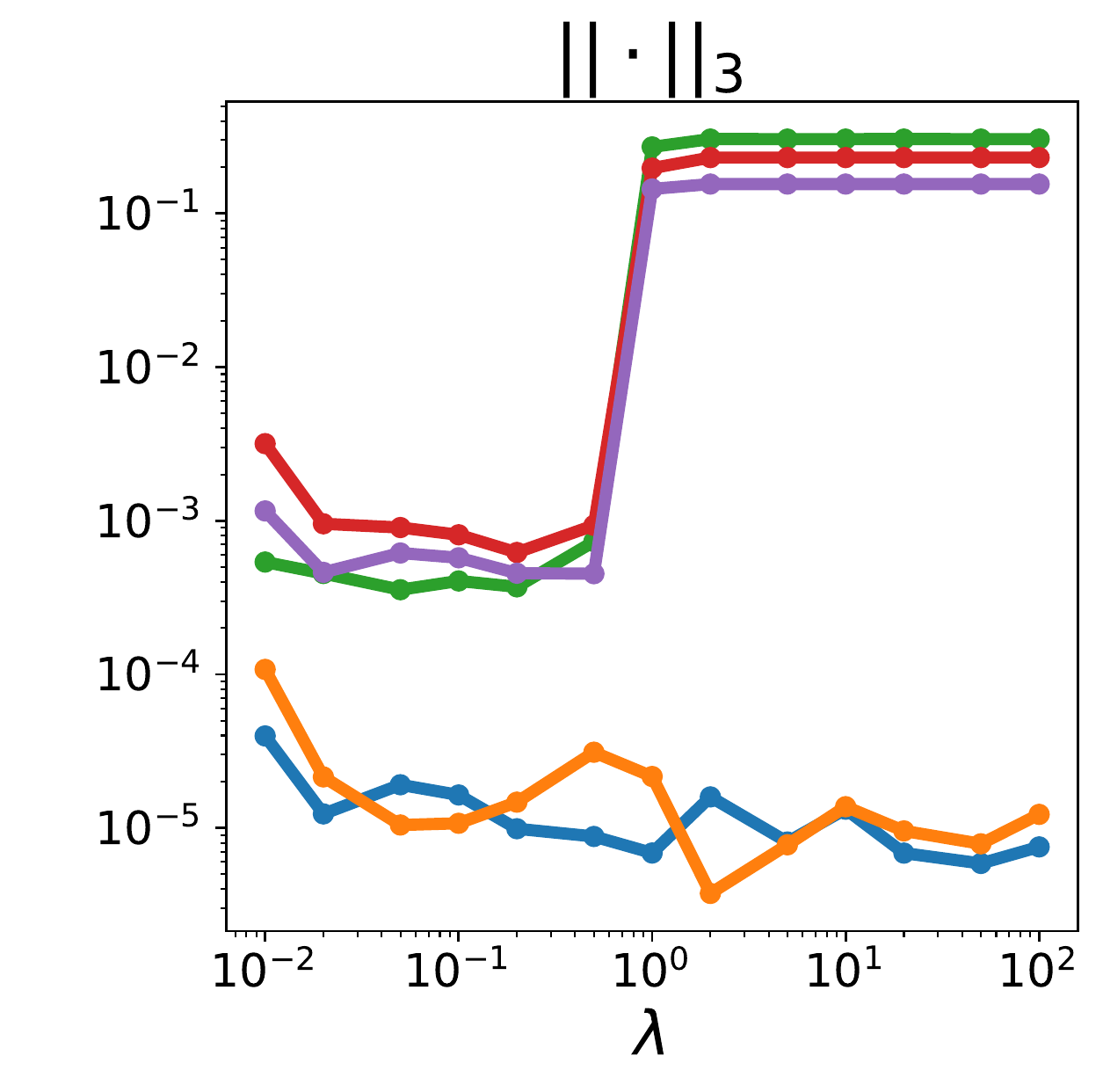}
		\caption{}
	\end{subfigure}
	\begin{subfigure}[b]{0.5\columnwidth}
		\includegraphics[trim=0cm 0.5cm 0cm 0cm,width=\textwidth]{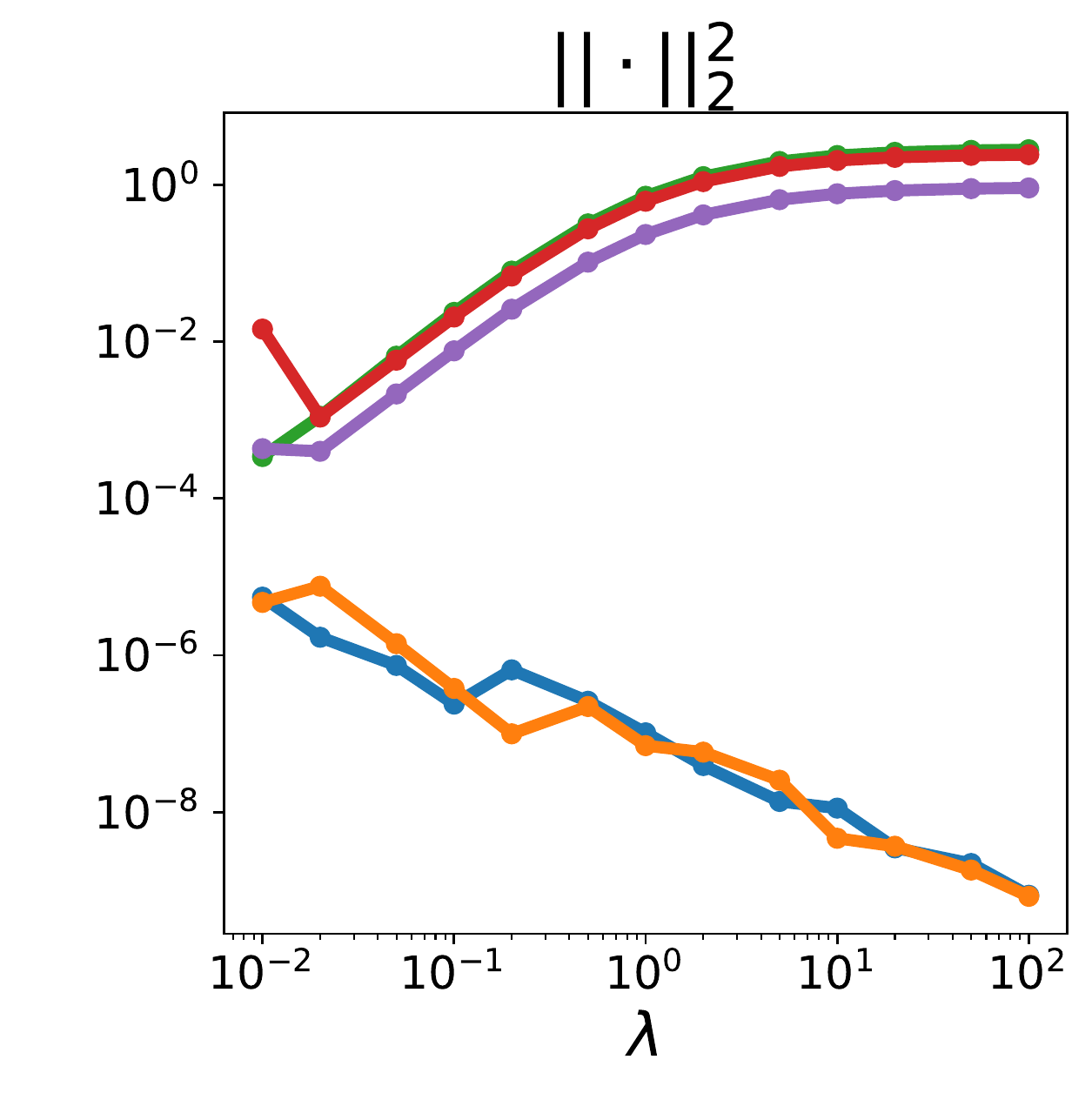}
		\caption{}
	\end{subfigure}

	\begin{subfigure}[b]{0.5\columnwidth}        
		\includegraphics[trim=0cm 0.5cm 0cm 0cm, width=\textwidth]{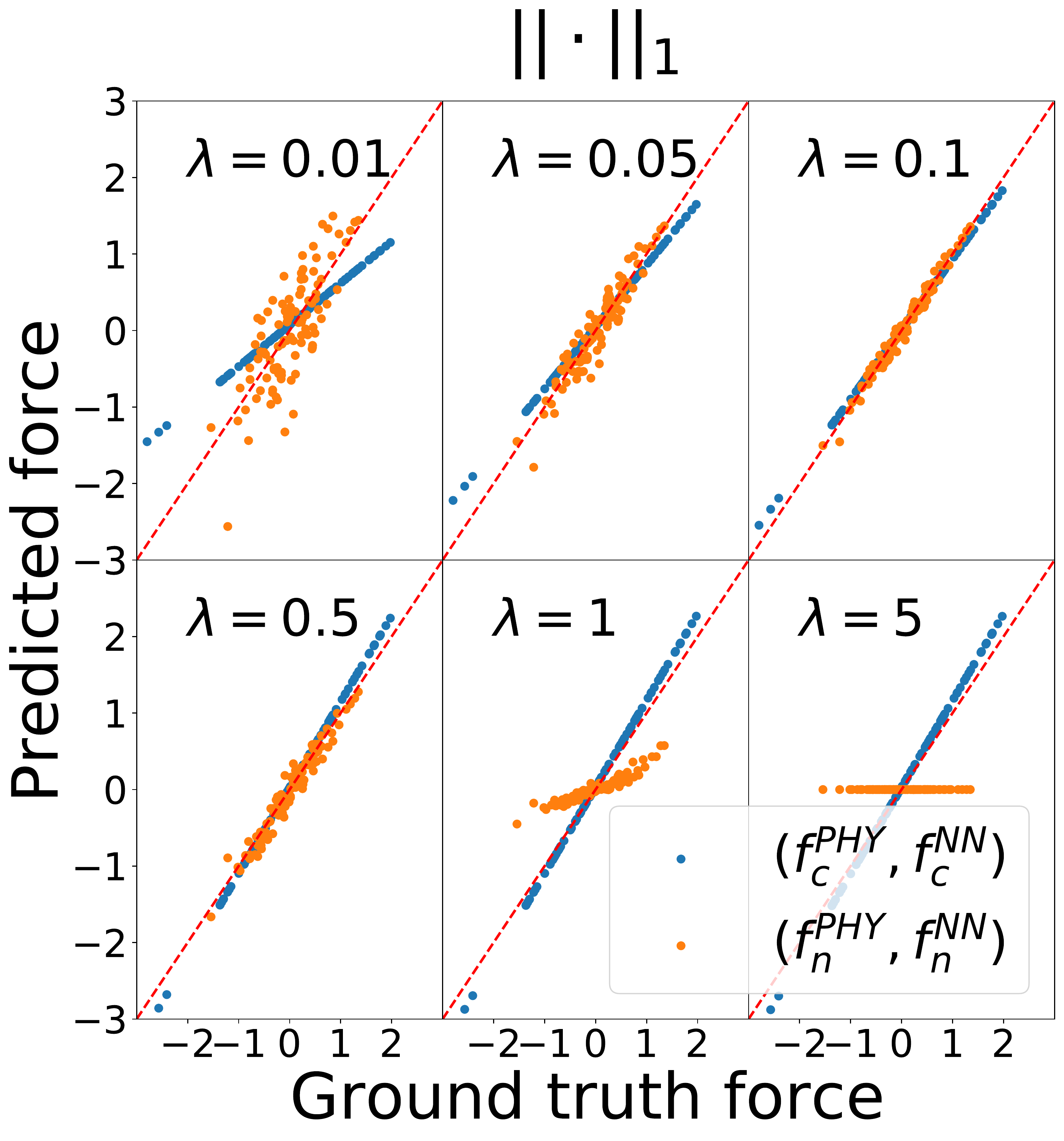}
		\caption{}
	\end{subfigure}
	\begin{subfigure}[b]{0.5\columnwidth}
		\includegraphics[trim=0cm 0.5cm 0cm 0cm,width=\textwidth]{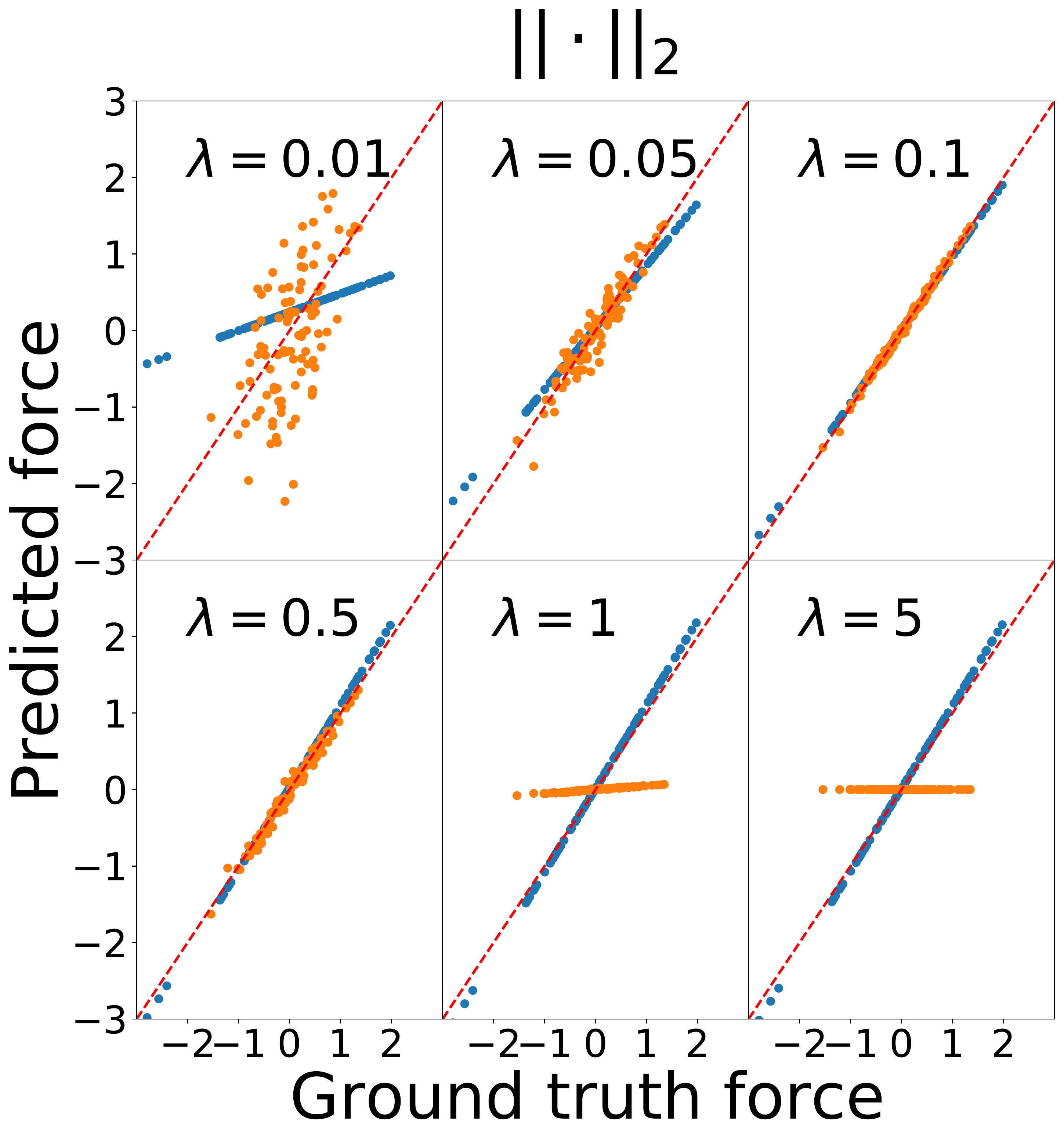}
		\caption{}
	\end{subfigure}
	\begin{subfigure}[b]{0.5\columnwidth}        
		\includegraphics[trim=0cm 0.5cm 0cm 0cm, width=\textwidth]{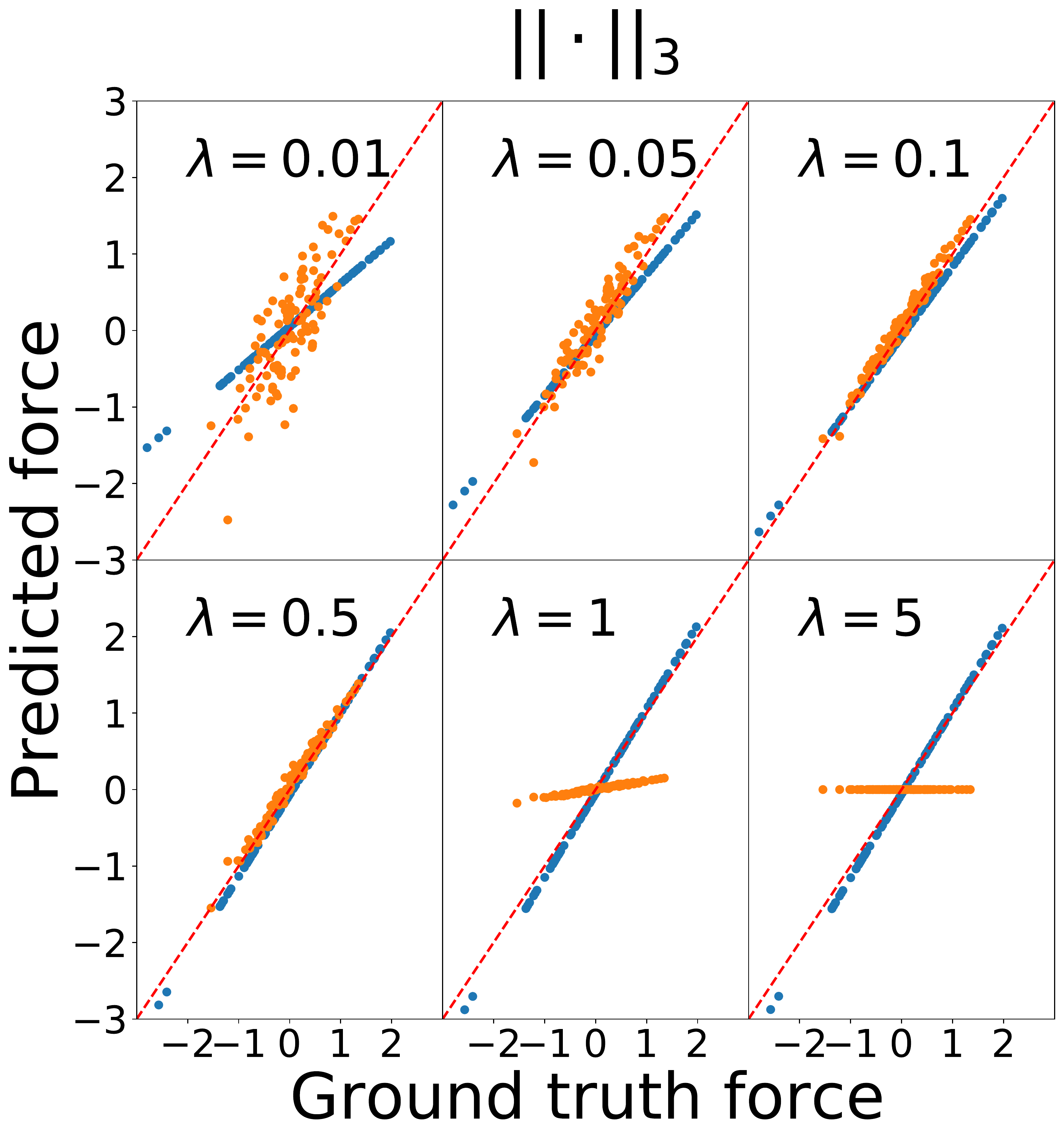}
		\caption{}
	\end{subfigure}
	\begin{subfigure}[b]{0.5\columnwidth}
		\includegraphics[trim=0cm 0.5cm 0cm 0cm,width=\textwidth]{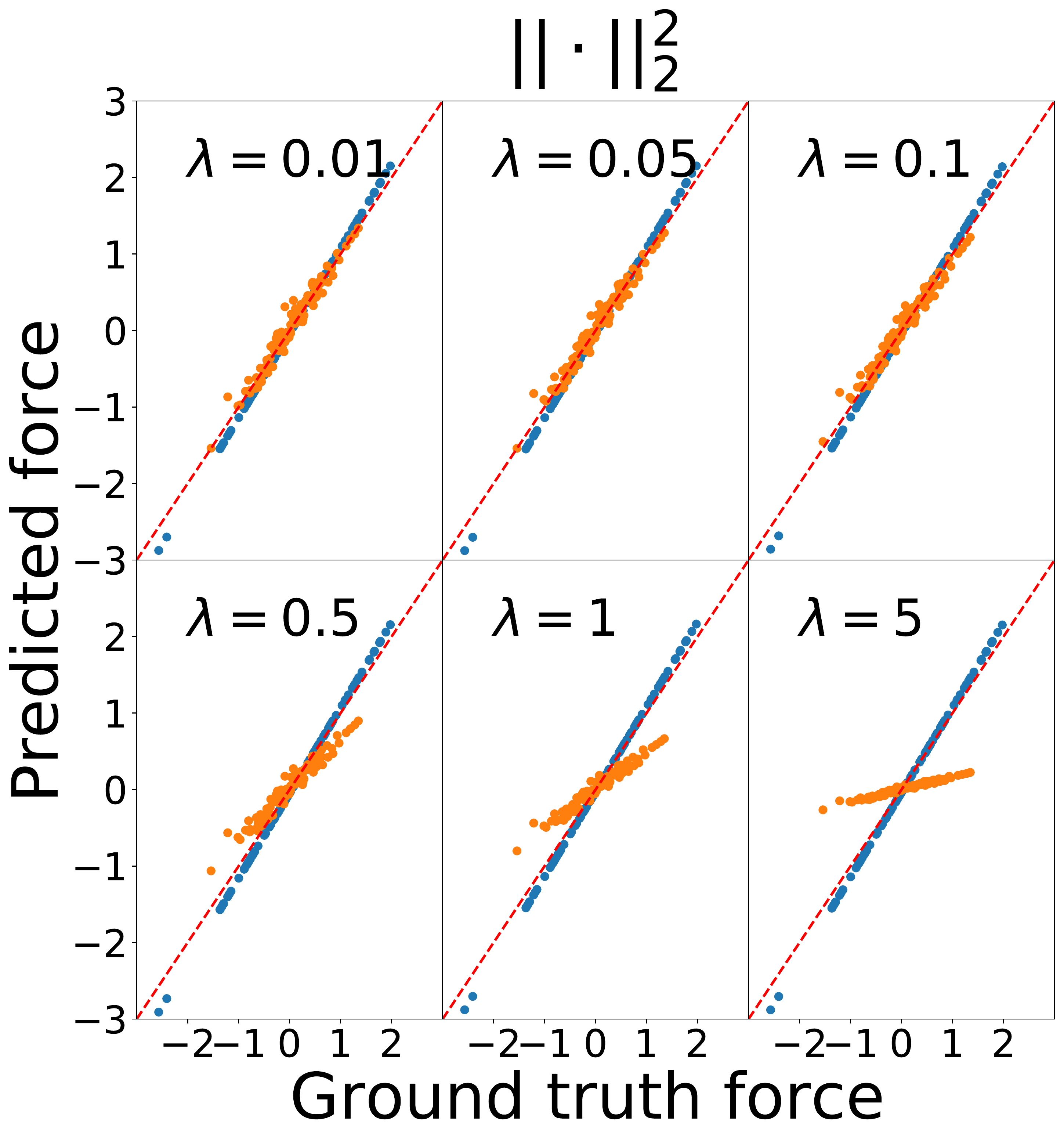}
		\caption{}
	\end{subfigure}
	\vskip -0.3cm
	\caption{NNPhD is seen to behave robustly for $0.05\lesssim\lambda<1$ and $p\geq 1$. We test NNPhD on five examples (the first two are conservative, and last three are non-conservative). (a)-(d) prediction error $L_e$ as a function of $\lambda$ with different norms as loss function: for (a)-(c) $\|\cdot\|_p(p=1,2,3)$, non-conservative dynamics has an error jump at $\lambda=1$, while conservative dynamics does not. In (d), mean squared loss leads to a smooth phase transition for non-conservative dynamics; (e)-(h)  for the linear damping case $\ddot{q}=-q-\frac{1}{2}\dot{q}$, we show how $f^{NN}_{\mathrm{c}}$ and $f^{NN}_{\mathrm{n}}$ are aligned with $f_{\mathrm{c}}^{PHY}$ and $f_{\mathrm{n}}^{PHY}$ for different loss functions and different $\lambda$.}
	\label{fig:decomposition}
	\vskip -0.3cm
\end{figure*}

\subsection{Theory Verification and Algorithm Benchmarking}\label{sec:exp_decomp_subsec}

\begin{table*}[ht]
	\caption{Examples of Conservative and Non-conservative Dynamics}
	\begin{center}
		\begin{tabular}{|c|c|c|c|}
			\hline
			Classes & Model & Equation &  Lagrangian \\\hline
			\multirow{3}{*}{\tabincell{c}{Conservative \\ ($f_{\mathrm{c}}^{PHY}$)}}&Harmonic Oscillator (HO)&$\ddot{q}=-q$ & $\lag=\dot{q}^2/2-q^2/2$\\\cline{2-4}
			&Magnetic Field (MF)& \tabincell{c}{$\ddot{q}_1=\dot{q}_2$ \\ $\ddot{q}_2=-\dot{q}_1$}&\tabincell{c}{$\lag=(\dot{q}_1-q_2)^2/2$ \\ $+(\dot{q}_2+q_1)^2/2$} \\\cline{2-4}
			&Constant Gravity (CG)&$\ddot{q}=-1$& $\lag=\dot{q}^2/2-q$\\
			\hline
			\multirow{3}{*}{\tabincell{c}{Non-Conservative \\ ($f_{\mathrm{n}}^{PHY}$)}}&Linear Damping (LD)&$\ddot{q}=-\dot{q}$&\multirow{3}{*}{NA}\\\cline{2-3}
			&Constant Damping (CD)&$\ddot{q}=-\mathrm{sgn}(\dot{q})$& \\\cline{2-3}
			&Periodic Force (PF)&$\ddot{q}=\mathrm{sin}(t)$&\\
			\hline
		\end{tabular}
	\end{center}
	\label{tab:lagrangian_example}
	\vskip -0.6cm
\end{table*}

In this section, to better understand its algorithmic behavior, we test NNPhD on the six simple dynamical systems in physics in Table \ref{tab:lagrangian_example}: conservative examples involve a harmonic oscillator (HO), a magnetic field (MF)\footnote{Note that we refer to the magnetic force as conservative because it conserves energy, even though physicists customarily limit that term to velocity-independent forces that can be written as the gradient of a potential function.} and constant gravity (CG) and non-conservative examples include linear damping (LD), constant damping (CD) and a periodic force (PF). We combine these into five examples to obtain two conservative systems (HO+MF, HO+CG) and three non-conservative systems (HO+LD, HO+CD, HO+PF), whose dynamical equations are summarized in Appendix \ref{app:equations}. For each system, we train NNPhD with the ADAM optimizer for 2,000 iterations, using batch size 32, learning rate schedule  $\{0.01,0.001,0.0001,0.00001\}$ and 500 iterations for each learning rate.

We now explore how the performance of the NNPhD depends on the regularization coefficient $\lambda$ and norm index $p$ by testing $\lambda=\{0.01,0.02,0.05,0.1,0.2,0.5,1,2,5,10,20,50,100\}$ and $p=1,2,3$. % $\mu$ is supported on the trajectory $(\q(t),\qd(t))\ (0\leq t\leq t_0)$, which is determined by dynamic $f$, initial condition, and simulation time $t_0$. 
Instead of simulating trajectories to generate data as in previous sections, we compute $\qd=f(\q,\qd,t)$ at $N$ random points $(\q^{(i)},\qd^{(i)},t_i)$. We first generate all positions, velocities and times as independent Gaussian random variables with zero mean and unit standard deviation, then explore more complicated coverage in Section \ref{sec:data_quality_exp}. We generate $10^3$ training samples and $10^3$ testing samples $(\q,\qd,t)$. 
\begin{figure*}[htbp]
	\centering
	\includegraphics[trim = 0cm 1cm 0cm 0cm, width=0.7\linewidth]{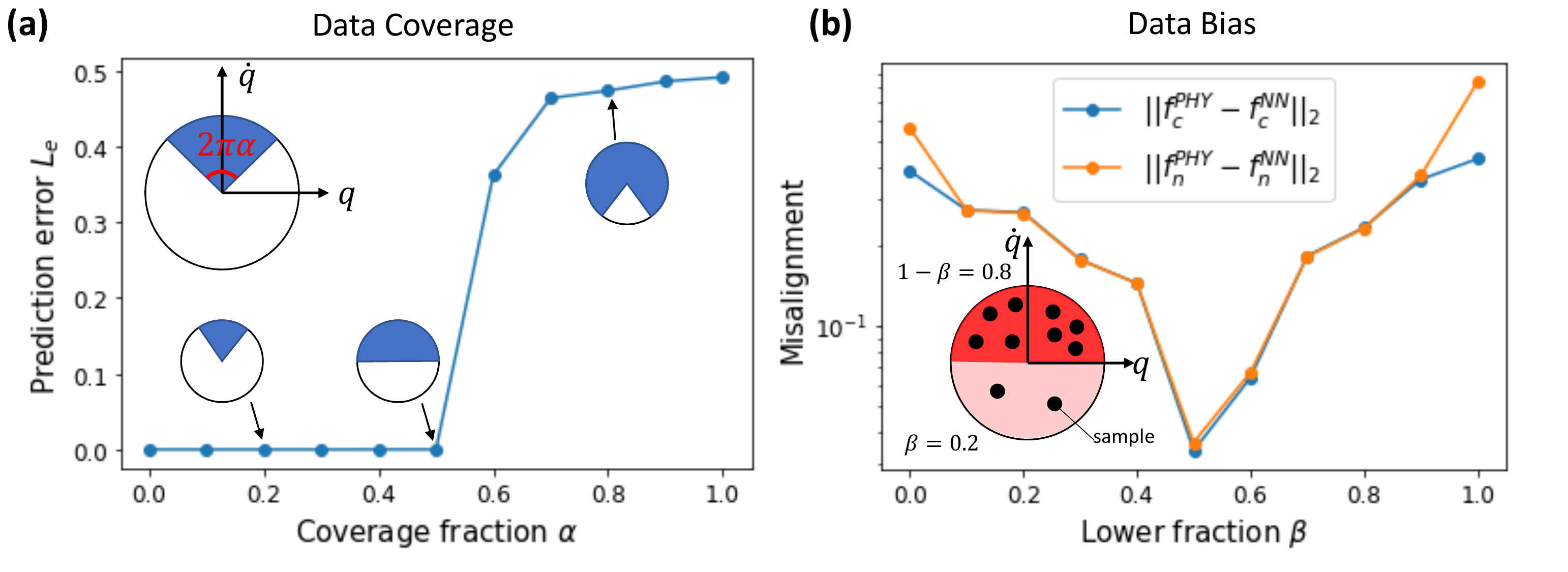}
	\caption{Dependence on data distribution parameters $\alpha$ and $\beta$. Low quality data might prevent new physics discovery via (a) incomplete data coverage and (b) biased data distribution.}
	\label{fig:dataquality}
\end{figure*}

{\bf How performance depends on $p$:} In Figure \ref{fig:decomposition}(a), we plot the dependence of the prediction error $L_e$ on $\lambda$ ($p=1$), again verifying the phase transition prediction from Section \ref{subsec:phase_transition}: The non-conservative systems (HO+LD, HO+CD, HO+PF) are seen to have a large error jump at $\lambda=1$ while, in contrast, {$L_e$ does not increase} at $\lambda=1$ for the conservative systems (HO+MF, HO+CG). In fact, HO+MF has even lower prediction error at larger $\lambda$, showing the advantage of employing a Lagrangian Neural Network as opposed to a black box for conservative systems. Figure \ref{fig:decomposition}(a) shows that NNPhD has the ability to distinguish between conservative and non-conservative dynamics by looking at prediction loss around $\lambda=1$, \ie, a sharp phase transition indicates non-conservative dynamics. The above observations also apply to Figure \ref{fig:decomposition}(b)(c) when $p=2$ and $p=3$. However Figure \ref{fig:decomposition}(d) shows that mean-squared-error loss (where the $L_2$-norm is squared) leads to a smooth transition, known as second-order phase transition in physics.

{\bf How performance depends on $\lambda$:} 
We then quantify how accurately the 
conservative and non-conservative components are modeled for different $\lambda$-values. 
Figure \ref{fig:decomposition}(e) shows our results for
the damped oscillator example $\ddot{q}=-q-\frac{1}{2}\dot{q}$, comparing $f_{\mathrm{c}}^{NN}$ with $f_{\mathrm{c}}^{PHY}=-q$ and $f_{\mathrm{n}}^{NN}$ with $f_{\mathrm{n}}^{PHY}=-\frac{1}{2}\dot{q}$. As Theorem \ref{thm:1} suggests, we observe that (1) when $\lambda>1$, $f^{NN}_{\mathrm{n}}$ predicts $0$ while $f^{NN}_{\mathrm{c}}\approx f_{\mathrm{c}}^{PHY}$; (2) when $0.05\lesssim\lambda<1$, $f^{NN}_{\mathrm{c}}\approx f^{PHY}_{\mathrm{c}}$ and $f^{NN}_{\mathrm{n}}\approx f^{PHY}_{\mathrm{n}}$; (3) when $\lambda\lesssim 0.05$, although in theory it behaves similarly to (2), a small $\lambda$ does not have much incentive to penalize the  black box, which therefore absorbs part of the conservative component. Figure \ref{fig:decomposition}(f)(g)(h) show that the alignments between the ground truth components and the the predictions from NNPhD are quite robust for different choices of loss function.

\subsection{Physics Discovery Requires High-Quality Data}\label{sec:data_quality_exp}

Although NNPhD does not assume any data distribution to achieve the decomposition, we will now see that NNPhD can only learn to accurately decompose the force into conservative and non-conservative parts if the 
data has high quality, specifically, if the data distribution has 
(1) adequate coverage of the state space $\x=(\q,\qd)$ and (2) is unbiased.

{\bf Incomplete data coverage:}
We now explore the situation where data points are only accessible in a pie-shaped subset of space covering an angular fraction of $\alpha\in [0,1]$, as illustrated in Figure \ref{fig:dataquality}(a). We consider the 1D constant damped oscillator $\ddot{q}=-q-\frac{1}{2}\mathrm{sgn}(\dot{q})$, train NNPhD with $\lambda=10$ on datasets with different fractions $\alpha$ and calculate the prediction loss $L_e$. Recall that when $\lambda=10$, a high prediction error $L_e$ is a sign of non-conservation. Figure \ref{fig:dataquality}(a) shows that when $\alpha\leq 0.5$, no samples are generated in the lower half plane (where $\dot{q}<0$), then the prediction error $L_e$ is nearly zero, revealing no sign of non-conservation. For $\alpha>0.5$, on the other hand, NNPhD has a large $L_e$, revealing the non-conservative nature of the damping force. This observation makes physical sense since, if only $\dot{q}>0$ samples are observed, the damping force acts as a constant conservative force (like gravity) which can be included as a $(-\frac{1}{2} q)$ term in a Lagrangian $\lag=\frac{1}{2}\dot{q}^2-\frac{1}{2}q^2-\frac{1}{2} q$, making the dynamics appear energy conserving.

{\bf Imbalanced data distribution} Even in the case when data is available everywhere in all relevant parts of phase space, the data set can still be imbalanced, \eg, contain more $\dot{q}>0$ samples than $\dot{q}<0$ ones. Figure \ref{fig:dataquality}(a) show that this is not a sever problem in the sense that it does not preclude us from identifying the \textit{existence} of non-conservative dynamics, since the presence of since merely a few samples with $\dot{q}<0$ suffices to give a clear signal of non-conservation. However, such imbalance may harm the accuracy of our decomposition. We consider the linear damped oscillator $\ddot{q}=-q-\frac{1}{2}\dot{q}$ where a fraction $\beta$ of the data is in the upper half plane while the remaining fraction $1-\beta$ is in the lower half-plane. We set $\lambda=0.5$, train on datasets with different $\beta$ and compare learned conservative and non-conservative force fields with ground truth. We found the learned functions $f_{\mathrm{c}}^{NN}$ and $f_{\mathrm{n}}^{NN}$ are not necessarily aligned with the ground truth decomposition  $f_{\mathrm{c}}^{PHY}=-q$ and $f_{\mathrm{n}}^{PHY}=-\frac{1}{2}\dot{q}$. 
We define misalignment as
\begin{equation}
    \begin{aligned}
      &m_{\mathrm{c}}=(\frac{1}{nN}\sum_{i=1}^N ||f_{\mathrm{c}}^{NN}(\x^{(i)})-f_{\mathrm{c}}^{PHY}(\x^{(i)})||^2)^{\frac{1}{2}},\\
      &m_{\mathrm{n}}=(\frac{1}{nN}\sum_{i=1}^N ||f_{\mathrm{n}}^{NN}(\x^{(i)},t^{(i)})-f_{\mathrm{n}}^{PHY}(\x^{(i)},t^{(i)})||^2)^{\frac{1}{2}}.
    \end{aligned}
\end{equation}
Figure \ref{fig:dataquality}(b)  shows this misalignment as a function of $\beta$, revealing a minimum with nearly zero  misalignment for the $\beta=0.5$ case when the data is balanced.
In summary, these last numerical experiments show that 
high-quality data is important for new physics discovery, regardless of whether the data is analyzed by intelligent human scientists or machine learning.

\section{Conclusion}\label{sec:conclusion}
We have presented the Neural New-physics Detector (NNPhD), a method for decomposing a general force field into components that do and do not conserve energy. We showed that NNPhD was able to do this robustly for a series of physical examples without access to symbolic equations, providing clear evidence of the existence of conservation-violating new physics. We also found that NNPhD could extrapolate time series more accurately than both LNN and black-box neural networks.
As ever-larger science and engineering datasets become available for dynamical systems, we hope that NNPhD will help enable more accurate prediction as well as aid discovery of interesting new phenomena.

{\bf Acknowledgements} We thank Yuanqi Du and Jieyu Zhang for helpful discussions, and the Center for Brains, Minds, and Machines (CBMM) for hospitality. This work was supported by the Casey and Family Foundation, the Foundational Questions Institute, the Rothberg Family Fund for Cognitive Science, and AI Institute for Artificial Intelligence and Fundamental Interactions (IAIFI) through NSF Grant No. PHY-2019786.

\bibliography{nnphd}

\clearpage

\onecolumngrid

\appendix

\section{Toy Example details}\label{app:equations}

For each dynamical system, the left hand side is $\qdd$, and right hand side is physical ground truth where conservative and non-conservative dynamics is explicitly separated as $\{f_{\mathrm{c}}^{PHY}(\q,\qd)\}+\{f_{\mathrm{n}}^{PHY}(\q,\qd,t)\}$.

{\bf HO+MF} ($k=B=1$):
\begin{equation}
	\begin{pmatrix}
		\ddot{x}\\
		\ddot{y}
	\end{pmatrix}=\{-k
	\begin{pmatrix}
		x\\
		y
	\end{pmatrix}+B
	\begin{pmatrix}
		\dot{y}\\
		-\dot{x}
	\end{pmatrix}\}+\{
	\begin{pmatrix}
		0\\
		0
	\end{pmatrix}
	\}
\end{equation}

{\bf HO+CG} ($k=g=1$)
\begin{equation}
	\ddot{x} = \{-kx-g\}+\{0\}
\end{equation}

{\bf HO+LD} ($k=1,\gamma=\frac{1}{2}$)
\begin{equation}
	\ddot{x} = \{-kx\}+\{-\gamma\dot{x}\}
\end{equation}

{\bf HO+CD} ($k=1,\gamma=\frac{1}{2}$)
\begin{equation}
	\ddot{x} = \{-kx\}+\{-\gamma \mathrm{sgn}(\dot{x})\}
\end{equation}

{\bf HO+PF} ($k=1,a=\frac{1}{2}$)
\begin{equation}
	\ddot{x} = \{-kx\}+\{a \mathrm{sin}(t)\}
\end{equation}

\section{Neural network training details}\label{app:training_details}

We parameterize both the our LNN (conservative) and our UAN (non-conservative) force models as non-weight-sharing fully connected feedforward neural networks with two hidden 200-neuron layers. The LNN uses a mixture of softplus and quadratic activation (see Appendix~\ref{app:training_trick} for details) and has Eq.~(\ref{eq:AL}) hard-coded right before outputting $f_{\mathrm{c}}^{NN}$, while the UAN uses LeakyReLU activation (with negative slope $\alpha=0.2$) and does not involve in any other inductive biases. 

We measure the performance of NNPhD for
\beq{lambdaEq}
\lambda\in\{.01,.02,.05,.1,.2,.5,1,2,5,10,20,50,100\}
\eeq
by first initializing and training NNPhD with $\lambda=0.01$ using the ADAM optimizer with 
learning rate $\{10^{-2},10^{-3},10^{-4},10^{-5}\}$ for 2000 steps (500 steps for each learning rate), and iteratively increasing $\lambda$ and train for 2000 steps for each new $\lambda$-value. 
The model parameters of LNN and UAN are updated simultaneously.

\section{Tricks to Boost LNN Training}\label{app:training_trick}

\begin{figure}
	\centering
	\includegraphics[width=0.4\linewidth]{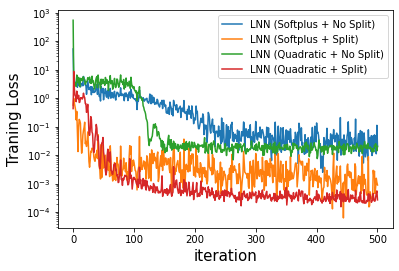}
	\caption{Tricks to boost LNN training}
	\label{fig:training_trick}
\end{figure}

As mentioned in \cite{cranmer2020lagrangian}, LNN is unstable and inefficient to train with traditional initializations in ML. As a result, expensive grid search of proper initializations is required. We propose two simpler tricks that have some improvements and are easy to implement. We use the example of a harmonic oscillator. The Lagrangian $\lag=\frac{1}{2}\dot{q}^2-\frac{1}{2}q^2$ contains only quadratic terms. We build a two hidden-layer networks with width 4-200-200-2.

{\bf Activation Trick}: \cite{cranmer2020lagrangian} uses Softplus as activation function, which is general but inefficient to represent a quadratic function. However the quadratic function is common and useful in physics, we propose to divide neurons into two groups, where one group uses Softplus as activation, and the other group uses quadractic function as activation.

{\bf Split Trick}:
One of instability when forwarding LNN comes from inversion of $\nabla_{\qd}\nabla_{\qd}\lag$. In physical terms, $\nabla_{\qd}\nabla_{\qd}\lag$ represents a mass scalar (matrix) which is positive (positive definite). However this constraint is not explicitly embedded to LNN, leading to training instabilities. We split $\lag$ into two parts:
\begin{equation}
	\lag=\lag_1+\lag_2=\lag_{NN}+\frac{1}{2}a\qd^T\qd
\end{equation}
where $\lag_1$ is learned by LNN, while $\lag_2$ is a fixed quadratic term (we choose $a=1$). At initializations when $\lag_{NN}\approx 0$, $\lag\approx \frac{1}{2}a\qd^T\qd$ is positive definite.

To test how the proposed two tricks operate, we implement four models in Figure \ref{fig:training_trick} to fit 1D harmonic oscillator: Softplus or quadratic activation, and w/wo the split trick. The best performance one is the LNN using quadratic activation and the split trick.

\section{Lagrangian mechanics for machine learning readers}\label{app:ml_lag}

For readers whose background is primarily in machine learning rather than physics, this section provides a brief review of the Lagrangian mechanics formalism that we use in this paper.

\textit{Conservative dynamics} describes a dynamic where there exist conserved quantities (energy, momentum, angular momentum etc). Conservation laws are important in physics, as it corresponds to symmetries of our mother nature, according to Noether's theorem~\cite{hanc2004symmetries}. In particular, energy conservation is equivalent to time translational symmetry. To describe dynamics that conserves energy, physicists employ (time-independent) \textit{Lagrangian} or \textit{Hamiltonian} formulation. Since our work and prior work Lagrangian Neural Network (LNN)~\cite{cranmer2020lagrangian} are based on Lagrangian mechanics, we provide a brief introduction here.

The Lagrangian formalism models a classical physics system with trajectory $\x(t)=(\q,\qd)$ that begins in one state $\x(t_0)$ and ends up in another state $\x(t_1) (t_1>t_0)$, where $\q$ and $\qd$ are called the generalized coordinates and velocities respectively. There are many paths that these states might take as they pass from $\x(t_0)$ to $\x(t_1)$, and Lagrangian mechanics tells that there is only one path that the physical system will take, i.e., the path that minimizes $\int_{t_0}^{t_1}(T(\q(t),\qd(t))-V(\q(t),\qd(t)))dt$, where $T$ is kinetic energy and $V$ is the potential energy. The term $\lag(\q,\qd)\equiv T(\q,\qd)-V(\q,\qd)$ is called Lagrangian and the path (trajectory) of the system is determined by \textit{Euler-Lagrange equation}:
	$$\frac{d}{dt}\nabla_{\qd} \lag=\nabla_{\q}\lag.$$

Based on the formulas in Lagrangian Neural Network (LNN)~\cite{cranmer2020lagrangian}, Euler-Lagrange equation $\frac{d}{dt}\nabla_{\qd} \lag=\nabla_{\q}\lag$ can be rewritten by {applying} a chain rule $\frac{d}{dt}\nabla_{\qd}^T\lag=(\nabla_{\qd}\nabla_{\qd}^T\lag)\qdd +(\nabla_{\q}\nabla_{\qd}\lag)\qd$ resulting in:
\begin{equation}
	\qdd = (\nabla_{\qd} \nabla_{\qd}^T\lag)^{-1}(\nabla_{\q}\lag-(\nabla_{\q}\nabla_{\qd}^T\lag)\qd)
\end{equation}

One inductive bias brought by Lagrangian mechanics is that Eq.~(\ref{eq:AL}) describes \textbf{conservative physical dynamics}. That is, the energy function defined as
\begin{equation}
\label{eq: definition_H}
	H(\q,\qd) = \nabla_{\qd}^T\lag(\q,\qd)\qd-\lag(\q,\qd)
\end{equation}
is constant along a trajectory $(\q(t),\qd(t))$ driven by Eq.~(\ref{eq:AL}). The proof of $H(\q,\qd)$ conservation can be found in standard physics textbooks~\cite{goldstein2002classical} and is included here for completeness.

\begin{lemma}\label{lemma:lag2conserve}
	Given a Lagrangian $\lag(\q,\qd)$, the energy defined in Eq.~(3) is conserved along the trajectory $(\q(t),\qd(t))$ driven by Eq.~(2).
\end{lemma}
\begin{proof}
	Invoke the chain rule one obtains the time derivative of $\lag$:
	\begin{equation}
		\frac{d\lag}{dt}=\qd^T\nabla_{\q}\lag+\qdd^T\nabla_{\qd}\lag
	\end{equation}
	Eq.~(4) is equivalent to Euler-Lagrangian equation $\frac{d}{dt}\nabla_{\qd}\lag=\nabla_{\q}\lag$, so we replace $\nabla_{\q}\lag$ with $\frac{d}{dt}\nabla_{\qd}\lag$:
	\begin{equation}
		\frac{d\lag}{dt}=\qd^T\frac{d}{dt}\nabla_{\qd}\lag+\qdd^T\nabla_{\qd}\lag=\frac{d}{dt}(\qd^T\nabla_{\qd}\lag)\longrightarrow \frac{dH}{dt}\equiv \frac{d}{dt}(\qd^T\nabla_{\qd}\lag-\lag)=0
	\end{equation}
\end{proof}

Since not all physical system conserves energy, Lagrangian mechanics is insufficient to describe non-conservative dynamics, motivating the design of NNPhD framework. We prove that linear damp example is non-conservative, i.e., it cannot be represented by Lagrangian mechanics.
\begin{lemma}
	Let function $f:\mathbb{R}^2\rightarrow\mathbb{R}$ be defined as $f(\q,\qd)=c\qd$, where $c$ can be any real non-zero constant. Then, $f$ cannot be represented by Eq.~(\ref{eq:AL}) for any Lagrangian $\lag\in D^2(\q,\qd)$ ($D^2(\q,\qd)$ is the function space consisting of all twice-differentiable functions with respect to $(\q, \qd)$).
\end{lemma}
\begin{proof}
	We prove the claim by reduction to absurdity. Suppose there exists a Lagrangian $\lag\in D^2(\q,\qd) $, such that, 
	\begin{equation}
		\label{eq:representaition_qd}
		c\qd=\left(\frac{\partial^2 \lag}{\partial \qd^2}(\q,\qd)\right)^{-1}\left(\frac{\partial \lag}{\partial \q}(\q,\qd)-\left(\frac{\partial^2 \lag}{\partial \qd\partial \q}(\q,\qd)\right) \qd\right).
	\end{equation}
	\\By multiplying $\left(\frac{\partial^2 \lag}{\partial \qd^2}(\q,\qd)\right)$ to both sides of eq. (\ref{eq:representaition_qd}), we have 
	\begin{equation*}
		c\left(\frac{\partial^2 \lag}{\partial \qd^2}(\q,\qd)\right)\qd=\frac{\partial \lag}{\partial \q}(\q,\qd)-\left(\frac{\partial^2 \lag}{\partial \qd\partial \q}(\q,\qd)\right) \qd,
	\end{equation*}
	which by eq. (\ref{eq: definition_H}) further leads to 
	\begin{equation}
		\label{eq:relationship_q_qd_gradient}
		c\frac{\partial H }{\partial \qd}(\q,\qd)+\frac{\partial H }{\partial \q}(\q,\qd)=0.
	\end{equation}
	\\By variable substitution, let $H(\q,\qd)=g(c\q+\qd,\qd-c\q)$. By eq. (\ref{eq:relationship_q_qd_gradient}), 
	\begin{align*}
	    \frac{\partial g(c\q+\qd,\qd-c\q)}{\partial (c\q+\qd)}= &\frac{\partial H(\q,\qd)}{\partial \q}\frac{\partial \q}{\partial (c\q+\qd)}+\frac{\partial H(\q,\qd)}{\partial \qd}\frac{\partial\qd}{\partial (c\q+\qd)}
	    \\
	    =&\frac{1}{2c}\frac{\partial H(\q,\qd)}{\partial \q}+\frac{1}{2}\frac{\partial H(\q,\qd)}{\partial \qd}=0.
	\end{align*}
	Therefore, $H(\q,\qd)$ is invariant of $c\q+\qd$ and only relies on the value of $\qd-c\q$. Thus, we can further abbreviate $H(\q,\qd)$ as $g(\qd-c\q)$.
	\\
	On the other hand, by eq. (\ref{eq: definition_H}), 
	\begin{equation*}
	    g(-c\q)=g(0-c\q)=H(\q,\qd)=-\lag(\q,0).
	\end{equation*}
	Therefore,
	\begin{align*}
		\left.\frac{\partial \lag(\q,\qd)}{\partial \qd}\right|_{\qd=0}&=\lim_{\qd\rightarrow 0} \frac{\lag(\q,\qd)-\lag(\q,0)}{\qd }
		\\ 
		&=\lim_{\qd\rightarrow 0} \frac{\qd \frac{\partial \lag(\q,\qd)}{\partial\qd}(\q,\qd)-g(\qd-c\q)+g(-c\q)}{\qd}
		\\
		&\overset{(*)}{=}\left.\frac{\partial \lag(\q,\qd)}{\partial \qd}\right|_{\qd=0}- g'(-c\q),
	\end{align*}
	where eq. $(*)$ is due to that $ \frac{\partial \lag(\q,\qd)}{\partial\qd}(\q,\qd)$ is differentiable (thus continuous).
	\\
	Therefore, we have $g'(\q)=0$ for any $\q$, which further leads to $H(\q,\qd)$ is a constant function and 
	\begin{equation*}
		c\qd=-\left(\frac{\partial^2 \lag}{\partial \qd^2}(\q,\qd)\right)^{-1} \frac{\partial H(\q,\qd)}{\partial \q}=0.
	\end{equation*}
	\\
	The proof is completed since $c\ne 0$.
\end{proof}

\section{Learning perspectives of Section \ref{sec:problem_formulation}}\label{sec:learning_perspective}

We describe the learning task based on the force decomposition in Section \ref{sec:problem_formulation}. Given samples $(\q^{(i)},\qd^{(i)},t^{(i)};\qdd^{(i)}), i=1,\cdots, N$ that are uniformly drawn from the trajectories of dynamic $\qdd=f(\q,\qd,t)$ with $t\in[0, T]$ or a given distribution $\mu$, we aim to learn both $f_{\mathrm{c}}$ and $f_{\mathrm{n}}$ from data. Because the ground-truth dynamic and its vector space are unknown, we need to select a model space $(\mathcal{G},\|\cdot\|_{\mathcal{G}})$ which is also a normed vector space to find the best model in it. For this learning problem, we learn the model pair $(f_{\mathrm{c}}^{NN},f_{\mathrm{n}}^{NN})$ simultaneously by solving the following constrained minimization problem
{\small\begin{align}
		& \min_{(f_{\mathrm{n}}^{NN},f_{\mathrm{c}}^{NN})}\mathcal{L}_S(f_{\mathrm{n}}^{NN},f_{\mathrm{c}}^{NN})=
		\frac{1}{N}\sum_{i=1}^N\|f_{\mathrm{c}}^{NN}(\q^{(i)},\qd^{(i)})+f_{\mathrm{n}}^{NN}(\q^{(i)},\qd^{(i)},t^{(i)})-\qdd^{(i)}\|_{\mathcal{G}}
		\nonumber\\
		&\quad\textit{s.t.}\quad f_{\mathrm{c}}^{NN}=\arg\min_{g\in\mathcal{G}_c}\frac{1}{N}\sum_{i=1}^N\|\qdd^{(i)}-g(\q^{(i)},\qd^{(i)})\|_{\mathcal{G}}\nonumber
\end{align}}

We make the following discussions on the learning task:
We denote the optimal models of the above optimization problem as $(f_{\mathrm{c}}^{NN*}, f_{\mathrm{n}}^{NN*})$. The interpolating prediction ability (which is measured by the gap between $\int_{0}^T\|f_{\mathrm{c}}^{NN*}+f_{\mathrm{n}}^{NN*}-\qdd\|dt$ and $\mathcal{L}_S(f_{\mathrm{c}}^{NN*},f_{\mathrm{n}}^{NN*})$) is determined by the approximation ability of $\mathcal{G}$ and the number of training data. As the number of training data $N$ increases, the gap will be smaller. As the approximation ability of $\mathcal{G}$ becomes stronger, the gap will be smaller.

\section{Theorem 1 (formal)}\label{sec:Theorem1_formal}

\begin{theorem}\label{thm:1}
	We suppose the ground-truth hypothesis space $(\mathcal{G},\|\cdot\|_p)$,
	Let $f_{\mathrm{c}}^{NN}(\q,\qd;\boldsymbol{w}_{\mathrm{c}})$ be the Lagrangian Neural Network with parameters $\boldsymbol{w}_{\mathrm{c}}$ in NNPhD
	framework, and $f_{\mathrm{n}}^{NN}(\q,\qd,t;\boldsymbol{w}_{\mathrm{n}})$ be the black box neural network with parameters $\boldsymbol{w}_{\mathrm{n}}$ in NNPhD framework. Assume the black box neural network can interpolate every continuous function of $(\q,\qd,t)$ at any $N$ points, i.e., for any dataset $\{(\q^{(i)},\qd^{(i)},t^{(i)})\}_{i=1}^N$ with distinguished elements, $\{\{f_{\mathrm{n}}^{NN}(\q^{(i)},\qd^{(i)},t^{(i)};\boldsymbol{w}_{\mathrm{n}})\}_{i=1}^N: \boldsymbol{w}_{\mathrm{n}}\in \mathbb{R}^{d_{\boldsymbol{w}_{\mathrm{n}}}}\}=\mathcal{R}^N$, where $d_{\boldsymbol{w}_{\mathrm{n}}}$ is the dimension of $\boldsymbol{w}_{\mathrm{n}}$. Then, given any continuous function $f$ and any norm $\Vert \cdot\Vert$ on function space $\mathcal{C}(\q,\qd,t)$, the following claim stands:
	\\
	(1) For $\lambda>1$, optimizing $L_{NNPhD}$ is equivalent to optimize $\Vert f_{\mathrm{c}}^{NN}(\cdot,\cdot;\boldsymbol{w}_{\mathrm{c}})-\qdd \Vert$
	while keeping $f_{\mathrm{n}}^{NN} (\cdot,\cdot,\cdot; \boldsymbol{w}_{\mathrm{n}})$ as zero, that is, 
	{\footnotesize\begin{align*}
		&\arg\min_{\boldsymbol{w}_{\mathrm{c}},\boldsymbol{w}_{\mathrm{n}}} \left(\frac{1}{N}\sum_{i=1}^N\|f_{\mathrm{c}}^{NN}(\q^{(i)},\qd^{(i)};\boldsymbol{w}_{\mathrm{c}})
		\right.
		+\left.f_{\mathrm{n}}^{NN}(\q^{(i)},\qd^{(i)},t^{(i)};\boldsymbol{w}_{\mathrm{n}})-\qdd^{(i)}\|^p\right)^{\frac{1}{p}} 
		+\lambda \left(\frac{1}{N}\Vert f_{\mathrm{n}}^{NN}(\q^{(i)},\qd^{(i)},t^{(i)};\boldsymbol{w}_{\mathrm{n}})\Vert^{p}\right)^{\frac{1}{p}}
		\\
		=&\left\{(\boldsymbol{w}_{\mathrm{c}},\boldsymbol{w}_{\mathrm{n}}):\boldsymbol{w}_{\mathrm{c}}\in \arg\min_{\boldsymbol{w}_{\mathrm{c}}}
		\quad\left(\frac{1}{N}\sum_{i=1}^N\|f_{\mathrm{c}}^{NN}(\q^{(i)},\qd^{(i)};\boldsymbol{w}_{\mathrm{c}})
		-\qdd^{(i)}\|^p\right)^{\frac{1}{p}} ,\quad f_{\mathrm{n}}^{NN}(\q^{(i)},\qd^{(i)},t^{(i)};\boldsymbol{w}_{\mathrm{n}})=0\right\}.  
	\end{align*}}
	\\
	(2) 
	For $0<\lambda<1$, optimizing $L_{NNPhD}$ is also equivalent to optimize $\Vert f_{\mathrm{c}}^{NN}(\cdot,\cdot;\boldsymbol{w}_{\mathrm{c}})-\qdd \Vert$
	while keeping $f_{\mathrm{n}}^{NN} (\cdot,\cdot,\cdot; \boldsymbol{w}_{\mathrm{n}})$ as $f-f_{\mathrm{c}}^{NN}(\cdot,\cdot;\boldsymbol{w}_{\mathrm{c}})$, that is, 
	{\footnotesize\begin{align*}
		&\arg\min_{\boldsymbol{w}_{\mathrm{c}},\boldsymbol{w}_{\mathrm{n}}} \left(\frac{1}{N}\sum_{i=1}^N\|f_{\mathrm{c}}^{NN}(\q^{(i)},\qd^{(i)};\boldsymbol{w}_{\mathrm{c}})
		\right.
		+\left.f_{\mathrm{n}}^{NN}(\q^{(i)},\qd^{(i)},t^{(i)};\boldsymbol{w}_{\mathrm{n}})-\qdd^{(i)}\|^p\right)^{\frac{1}{p}} 
		+\lambda \left(\frac{1}{N}\Vert f_{\mathrm{n}}^{NN}(\q^{(i)},\qd^{(i)},t^{(i)};\boldsymbol{w}_{\mathrm{n}})\Vert^{p}\right)^{\frac{1}{p}}
		\\
		=&\left\{(\boldsymbol{w}_{\mathrm{c}},\boldsymbol{w}_{\mathrm{n}}):\boldsymbol{w}_{\mathrm{c}}\in \arg\min_{\boldsymbol{w}_{\mathrm{c}}}
		\quad\left(\frac{1}{N}\sum_{i=1}^N\|f_{\mathrm{c}}^{NN}(\q^{(i)},\qd^{(i)};\boldsymbol{w}_{\mathrm{c}})
		-\qdd^{(i)}\|^p\right)^{\frac{1}{p}} ,\quad f_{\mathrm{n}}^{NN}(\q^{(i)},\qd^{(i)},t^{(i)};\boldsymbol{w}_{\mathrm{n}})=f(\q^{(i)},\qd^{(i)},t^{(i)})-f_{\mathrm{c}}^{NN}(\q^{(i)},\qd^{(i)};\boldsymbol{w}_{\mathrm{c}})\right\}.
	\end{align*}}
\end{theorem}

We will provide proof of Theorem \ref{thm:1}. We will actually show our results hold for general norms which include the discrete norm we use in Theorem \ref{thm:1}.  Concretely, Theorem \ref{thm:1} holds as a special case as the following theorem:
\begin{theorem}[Theorem \ref{thm:1}, extended to general norms]\label{thm:2}
	Let $f_{\mathrm{c}}^{NN}(\q,\qd;\boldsymbol{w}_{\mathrm{c}})$ be the Lagrangian Neural Network with parameters $\boldsymbol{w}_{\mathrm{c}}$ in NNPhD
	framework, and $f_{\mathrm{n}}^{NN}(\q,\qd,t;\boldsymbol{w}_{\mathrm{n}})$ be the black box neural network with parameters $\boldsymbol{w}_{\mathrm{n}}$ in NNPhD framework. Assume the black box neural network can represent every continuous function of $(\q,\qd,t)$ under the norm $\Vert \cdot\Vert$, i.e., $\{g(\q,\qd,t):\exists \boldsymbol{w}_{\mathrm{n}},\text{ }\Vert g(\q,\qd,t)- f_{\mathrm{n}}^{NN}(\q,\qd;\boldsymbol{w}_{\mathrm{n}})\Vert=0\}=\mathcal{C}(\q,\qd,t)$. Then, given any continuous function $f$, the following claim stands:
	\\
	(1) For $\lambda>1$, 
	\begin{align*}
		&\arg\min_{\boldsymbol{w}_{\mathrm{c}},\boldsymbol{w}_{\mathrm{n}}} \Vert f(\q,\qd,t)-f_{\mathrm{c}}^{NN}(\q,\qd;\boldsymbol{w}_{\mathrm{c}})
		-f_{\mathrm{n}}^{NN}(\q,\qd,t;\boldsymbol{w}_{\mathrm{n}})\Vert +\lambda \Vert f_{\mathrm{n}}^{NN}(\q,\qd,t;\boldsymbol{w}_{\mathrm{n}})\Vert
		\\
		=&\left\{(\boldsymbol{w}_{\mathrm{c}},\boldsymbol{w}_{\mathrm{n}}):\boldsymbol{w}_{\mathrm{c}}\in \arg\min_{\boldsymbol{w}_{\mathrm{c}}} \Vert f(\q,\qd,t)-f_{\mathrm{c}}^{NN}(\q,\qd;\boldsymbol{w}_{\mathrm{c}})\Vert,
		\quad \Vert f_{\mathrm{n}}^{NN}(\q,\qd,t;\boldsymbol{w}_{\mathrm{n}})\Vert=0\right\}.  
	\end{align*}
	(2) For $0<\lambda<1$,
	\begin{align*}
		&\arg\min_{\boldsymbol{w}_{\mathrm{c}},\boldsymbol{w}_{\mathrm{n}}} \Vert f(\q,\qd,t)-f_{\mathrm{c}}^{NN}(\q,\qd;\boldsymbol{w}_{\mathrm{c}})
		-f_{\mathrm{n}}^{NN}(\q,\qd,t;\boldsymbol{w}_{\mathrm{n}})\Vert +\lambda \Vert f_{\mathrm{n}}^{NN}(\q,\qd,t;\boldsymbol{w}_{\mathrm{n}})\Vert
		\\
		=&\left\{(\boldsymbol{w}_{\mathrm{c}},\boldsymbol{w}_{\mathrm{n}}):\boldsymbol{w}_{\mathrm{c}}\in \arg\min_{\boldsymbol{w}_{\mathrm{c}}} \Vert f(\q,\qd,t)-f_{\mathrm{c}}^{NN}(\q,\qd;\boldsymbol{w}_{\mathrm{c}})\Vert,
		\Vert f_{\mathrm{n}}^{NN}(\q,\qd,t;\boldsymbol{w}_{\mathrm{n}})-f(\q,\qd,t)+f_{\mathrm{c}}^{NN}(\q,\qd;\boldsymbol{w}_{\mathrm{c}})\Vert=0\right\}. 
	\end{align*}
\end{theorem}
\begin{proof}
	We prove the two cases above separately.
	\\
	(1) If $\lambda>1$, for any $\boldsymbol{w}_{\mathrm{c}}$ and $\boldsymbol{w}_{\mathrm{n}}$,
	\begin{align}
		\nonumber
		&\Vert f(\q,\qd, t)-f_{\mathrm{c}}^{NN}(\q,\qd;\boldsymbol{w}_{\mathrm{c}})-f_{\mathrm{n}}^{NN}(\q,\qd,t;\boldsymbol{w}_{\mathrm{n}})\Vert +\lambda \Vert f_{\mathrm{n}}^{NN}(\q,\qd,t;\boldsymbol{w}_{\mathrm{n}})\Vert
		\\
		\nonumber
		=&\Vert f(\q,\qd, t)-f_{\mathrm{c}}^{NN}(\q,\qd;\boldsymbol{w}_{\mathrm{c}})-f_{\mathrm{n}}^{NN}(\q,\qd,t;\boldsymbol{w}_{\mathrm{n}})\Vert+\Vert f_{\mathrm{n}}^{NN}(\q,\qd,t;\boldsymbol{w}_{\mathrm{n}})\Vert +(\lambda-1) \Vert f_{\mathrm{n}}^{NN}(\q,\qd,t;\boldsymbol{w}_{\mathrm{n}})\Vert
		\\
		\label{eq:triangle inequality}
		\overset{(*)}{\ge}& \Vert f(\q,\qd,t)-f_{\mathrm{c}}^{NN}(\q,\qd;\boldsymbol{w}_{\mathrm{c}})\Vert
		+(\lambda-1) \Vert f_{\mathrm{n}}^{NN}(\q,\qd,t;\boldsymbol{w}_{\mathrm{n}})\Vert.
	\end{align}
	Since $\lambda-1>0$, 
	\begin{align*}
		&\arg\min_{\boldsymbol{w}_{\mathrm{c}},\boldsymbol{w}_{\mathrm{n}}}\left(\Vert f( \q, \qd, t)-f_{\mathrm{c}}^{NN}(\q,\qd;\boldsymbol{w}_{\mathrm{c}})\Vert
		+(\lambda-1) \Vert f_{\mathrm{n}}^{NN}(\q,\qd,t;\boldsymbol{w}_{\mathrm{n}})\Vert\right)
		\\
		=&\left(\arg\min_{\boldsymbol{w}_{\mathrm{c}}} \Vert f(\q,\qd, t)-f_{\mathrm{c}}^{NN}(\q,\qd;\boldsymbol{w}_{\mathrm{c}})\Vert, \arg_{\boldsymbol{w}_{\mathrm{n}}}\Vert f_{\mathrm{n}}^{NN}(\q,\qd,t;\boldsymbol{w}_{\mathrm{n}})\Vert=0\right).
	\end{align*}
	\\
	For any $(\boldsymbol{w}^0_{\mathrm{c}},\boldsymbol{w}^0_{\mathrm{n}})$ where $\boldsymbol{w}^0_{\mathrm{c}}\in\arg\min_{\boldsymbol{w}_{\mathrm{c}}} \Vert f(\q,\qd, t)-f_{\mathrm{c}}^{NN}(\q,\qd;\boldsymbol{w}_{\mathrm{c}})\Vert$ and $\Vert f_{\mathrm{n}}^{NN}(\q,\qd;\w_{\mathrm{n}}^0)\Vert=0$, the equality of inequality $(*)$ of Eq. (\ref{eq:triangle inequality}) is obtained. Therefore,  
	\begin{equation*}
		(\boldsymbol{w}^0_{\mathrm{c}},\boldsymbol{w}^0_{\mathrm{n}})\in \arg\min_{\boldsymbol{w}_{\mathrm{c}},\boldsymbol{w}_{\mathrm{n}}} \Vert f(\q,\qd, t)-f_{\mathrm{c}}^{NN}(\q,\qd;\boldsymbol{w}_{\mathrm{c}})-f_{\mathrm{n}}^{NN}(\q,\qd,t;\boldsymbol{w}_{\mathrm{n}})\Vert +\lambda \Vert f_{\mathrm{n}}^{NN}(\q,\qd,t;\boldsymbol{w}_{\mathrm{n}})\Vert,
	\end{equation*}
	which further leads to
	\begin{align*}
		&\arg\min_{\boldsymbol{w}_{\mathrm{c}},\boldsymbol{w}_{\mathrm{n}}}\left( \Vert f(\q,\qd, t)-f_{\mathrm{c}}^{NN}(\q,\qd;\boldsymbol{w}_{\mathrm{c}})\Vert +(\lambda-1) \Vert f_{\mathrm{n}}^{NN}(\q,\qd,t;\boldsymbol{w}_{\mathrm{n}})\Vert\right)
		\\
		\subset &\arg\min_{\boldsymbol{w}_{\mathrm{c}},\boldsymbol{w}_{\mathrm{n}}}\left( \Vert f(\q,\qd,t)-f_{\mathrm{c}}^{NN}(\q,\qd;\boldsymbol{w}_{\mathrm{c}})-f_{\mathrm{n}}^{NN}(\q,\qd,t;\boldsymbol{w}_{\mathrm{n}})\Vert +\lambda \Vert f_{\mathrm{n}}^{NN}(\q,\qd,t;\boldsymbol{w}_{\mathrm{n}})\Vert\right),
	\end{align*}
	and
	\begin{align}
		\nonumber
		&\min_{\boldsymbol{w}_{\mathrm{c}},\boldsymbol{w}_{\mathrm{n}}}\left( \Vert f(\q,\qd,t)-f_{\mathrm{c}}^{NN}(\q,\qd;\boldsymbol{w}_{\mathrm{c}})-f_{\mathrm{n}}^{NN}(\q,\qd,t;\boldsymbol{w}_{\mathrm{n}})\Vert +\lambda \Vert f_{\mathrm{n}}^{NN}(\q,\qd,t;\boldsymbol{w}_{\mathrm{n}})\Vert\right)
		\\
		\label{eq:equivalence_minimum_lambda_ge_1}
		= &\min_{\boldsymbol{w}_{\mathrm{c}},\boldsymbol{w}_{\mathrm{n}}}\left( \Vert f(\q,\qd,t)-f_{\mathrm{c}}^{NN}(\q,\qd;\boldsymbol{w}_{\mathrm{c}})\Vert +(\lambda-1) \Vert f_{\mathrm{n}}^{NN}(\q,\qd,t;\boldsymbol{w}_{\mathrm{n}})\Vert\right).
	\end{align}
	Combining Eq. (\ref{eq:triangle inequality}), Eq. (\ref{eq:equivalence_minimum_lambda_ge_1}) further leads to 
	\begin{align*}
		&\arg\min_{\boldsymbol{w}_{\mathrm{c}},\boldsymbol{w}_{\mathrm{n}}}\left( \Vert f(\q,\qd,t)-f_{\mathrm{c}}^{NN}(\q,\qd;\boldsymbol{w}_{\mathrm{c}})-f_{\mathrm{n}}^{NN}(\q,\qd,t;\boldsymbol{w}_{\mathrm{n}})\Vert +\lambda \Vert f_{\mathrm{n}}^{NN}(\q,\qd,t;\boldsymbol{w}_{\mathrm{n}})\Vert\right)
		\\
		\subset &\arg\min_{\boldsymbol{w}_{\mathrm{c}},\boldsymbol{w}_{\mathrm{n}}}\left( \Vert f(\q,\qd,t)-f_{\mathrm{c}}^{NN}(\q,\qd;\boldsymbol{w}_{\mathrm{c}})\Vert +(\lambda-1) \Vert f_{\mathrm{n}}^{NN}(\q,\qd,t;\boldsymbol{w}_{\mathrm{n}})\Vert\right).
	\end{align*}
	The proof for $\lambda>1$ is completed.
	\\
	(2) If $\lambda<1$, for any $\boldsymbol{w}_{\mathrm{c}}$ and $\boldsymbol{w}_{\mathrm{n}}$, we decompose $\Vert f(\q,\qd,t)-f_{\mathrm{c}}^{NN}(\q,\qd;\boldsymbol{w}_{\mathrm{c}})-f_{\mathrm{n}}^{NN}(\q,\qd,t;\boldsymbol{w}_{\mathrm{n}})\Vert +\lambda \Vert f_{\mathrm{n}}^{NN}(\q,\qd,t;\boldsymbol{w}_{\mathrm{n}})\Vert$ as follows:
	\begin{align}
		\nonumber
		&\Vert f(\q,\qd,t)-f_{\mathrm{c}}^{NN}(\q,\qd;\boldsymbol{w}_{\mathrm{c}})-f_{\mathrm{n}}^{NN}(\q,\qd,t;\boldsymbol{w}_{\mathrm{n}})\Vert +\lambda \Vert f_{\mathrm{n}}^{NN}(\q,\qd,t;\boldsymbol{w}_{\mathrm{n}})\Vert
		\\
		\nonumber
		=&(\lambda+(1-\lambda))\Vert f(\q,\qd,t)-f_{\mathrm{c}}^{NN}(\q,\qd;\boldsymbol{w}_{\mathrm{c}})-f_{\mathrm{n}}^{NN}(\q,\qd,t;\boldsymbol{w}_{\mathrm{n}})\Vert+\lambda\Vert f_{\mathrm{n}}^{NN}(\q,\qd,t;\boldsymbol{w}_{\mathrm{n}})\Vert 
		\\
		\label{eq:triangle_inequality_lambda_le_1}
		\overset{(**)}{\ge}& \lambda\Vert f(\q,\qd,t)-f_{\mathrm{c}}^{NN}(\q,\qd;\boldsymbol{w}_{\mathrm{c}})\Vert
		+(1-\lambda) \Vert f(\q,\qd,t)-f_{\mathrm{c}}^{NN}(\q,\qd;\boldsymbol{w}_{\mathrm{c}})-f_{\mathrm{n}}^{NN}(\q,\qd,t;\boldsymbol{w}_{\mathrm{n}})\Vert,
	\end{align}
	where eq. $(**)$ is due to triangle inequality.
	\\
	On the other hand, for any fixed $\boldsymbol{w}_{\mathrm{c}}$, minimum of eq. (\ref{eq:triangle_inequality_lambda_le_1}) is obtained if and only if $\Vert f_{\mathrm{n}}^{NN}(\q,\qd,t;\boldsymbol{w}_{\mathrm{n}})-f(\q,\qd,t)+f_{\mathrm{c}}^{NN}(\q,\qd;\boldsymbol{w}_{\mathrm{c}})\Vert=0$, in which case equality of eq. $(**)$ is also obtained. Therefore, for a given $\boldsymbol{w}_{\mathrm{c}}$,
	\begin{align*}
		&\min_{\boldsymbol{w}_{\mathrm{n}}}\Vert f(\q,\qd,t)-f_{\mathrm{c}}^{NN}(\q,\qd;\boldsymbol{w}_{\mathrm{c}})-f_{\mathrm{n}}^{NN}(\q,\qd,t;\boldsymbol{w}_{\mathrm{n}})\Vert +\lambda \Vert f_{\mathrm{n}}^{NN}(\q,\qd,t;\boldsymbol{w}_{\mathrm{n}})\Vert
		\\
		=&\lambda\Vert f(\q,\qd,t)-f_{\mathrm{c}}^{NN}(\q,\qd;\boldsymbol{w}_{\mathrm{c}})\Vert,
	\end{align*}
	and 
	\begin{align}
		\nonumber
		&\arg\min_{\boldsymbol{w}_{\mathrm{n}}}\Vert f(\q,\qd,t)-f_{\mathrm{c}}^{NN}(\q,\qd;\boldsymbol{w}_{\mathrm{c}})-f_{\mathrm{n}}^{NN}(\q,\qd,t;\boldsymbol{w}_{\mathrm{n}})\Vert +\lambda \Vert f_{\mathrm{n}}^{NN}(\q,\qd,t;\boldsymbol{w}_{\mathrm{n}})\Vert
		\\
		\label{eq:lambda_le_1_def_w_n}
		=&\{\boldsymbol{w}_{\mathrm{n}}:\Vert f_{\mathrm{n}}^{NN}(\q,\qd,t;\boldsymbol{w}_{\mathrm{n}})-f(\q,\qd,t)+f_{\mathrm{c}}^{NN}(\q,\qd;\boldsymbol{w}_{\mathrm{c}})\Vert=0\}.
	\end{align}
	\\
	Since 
	\begin{align*}
		&\arg\min_{\boldsymbol{w}_{\mathrm{c}},\boldsymbol{w}_{\mathrm{n}}}\Vert f(\q,\qd,t)-f_{\mathrm{c}}^{NN}(\q,\qd;\boldsymbol{w}_{\mathrm{c}})-f_{\mathrm{n}}^{NN}(\q,\qd,t;\boldsymbol{w}_{\mathrm{n}})\Vert +\lambda \Vert f_{\mathrm{n}}^{NN}(\q,\qd,t;\boldsymbol{w}_{\mathrm{n}})\Vert
		\\
		=&\arg\min_{\boldsymbol{w}_{\mathrm{c}}}\min_{\boldsymbol{w}_{\mathrm{n}}}\Vert f(\q,\qd,t)-f_{\mathrm{c}}^{NN}(\q,\qd;\boldsymbol{w}_{\mathrm{c}})-f_{\mathrm{n}}^{NN}(\q,\qd,t;\boldsymbol{w}_{\mathrm{n}})\Vert +\lambda \Vert f_{\mathrm{n}}^{NN}(\q,\qd,t;\boldsymbol{w}_{\mathrm{n}})\Vert,
	\end{align*}
	by applying eq. (\ref{eq:lambda_le_1_def_w_n}), we finally have
	\begin{align*}
		&\arg\min_{\boldsymbol{w}_{\mathrm{c}},\boldsymbol{w}_{\mathrm{n}}}\Vert f(\q,\qd,t)-f_{\mathrm{c}}^{NN}(\q,\qd;\boldsymbol{w}_{\mathrm{c}})-f_{\mathrm{n}}^{NN}(\q,\qd,t;\boldsymbol{w}_{\mathrm{n}})\Vert +\lambda \Vert f_{\mathrm{n}}^{NN}(\q,\qd,t;\boldsymbol{w}_{\mathrm{n}})\Vert
		\\
		=&\left\{(\boldsymbol{w}_{\mathrm{c}},\boldsymbol{w}_{\mathrm{n}}):\boldsymbol{w}_{\mathrm{c}}\in \arg\min_{\boldsymbol{w}_{\mathrm{c}}} \lambda\Vert f(\q,\qd,t)-f_{\mathrm{c}}^{NN}(\q,\qd;\boldsymbol{w}_{\mathrm{c}})\Vert,
		\Vert f_{\mathrm{n}}^{NN}(\q,\qd,t;\boldsymbol{w}_{\mathrm{n}})-f(\q,\qd,t)+f_{\mathrm{c}}^{NN}(\q,\qd;\boldsymbol{w}_{\mathrm{c}})\Vert=0\right\}
		\\
		=&\left\{(\boldsymbol{w}_{\mathrm{c}},\boldsymbol{w}_{\mathrm{n}}):\boldsymbol{w}_{\mathrm{c}}\in \arg\min_{\boldsymbol{w}_{\mathrm{c}}} \Vert f(\q,\qd,t)-f_{\mathrm{c}}^{NN}(\q,\qd;\boldsymbol{w}_{\mathrm{c}})\Vert,
		\Vert f_{\mathrm{n}}^{NN}(\q,\qd,t;\boldsymbol{w}_{\mathrm{n}})-f(\q,\qd,t)+f_{\mathrm{c}}^{NN}(\q,\qd;\boldsymbol{w}_{\mathrm{c}})\Vert=0\right\}.
	\end{align*}
	The proof is completed.
\end{proof}

The above theorem describes the case that $f_{\mathrm{n}}^{NN}(\q,\qd,t;\boldsymbol{w}_{\mathrm{n}})$ can represent every continuous function. However, in practice, the black box neural network can only access functions close to the original solution. In this general case, we instead have
\begin{corollary}\label{cor:theorem1_extension}
	Let $f_{\mathrm{c}}^{NN}(\q,\qd;\boldsymbol{w}_{\mathrm{c}})$ be the Lagrangian Neural Network with parameters $\boldsymbol{w}_{\mathrm{c}}$ in NNPhD
	framework, and $f_{\mathrm{n}}^{NN}(\q,\qd,t;\boldsymbol{w}_{\mathrm{n}})$ be the black box neural network with parameters $\boldsymbol{w}_{\mathrm{n}}$ in NNPhD framework. Assume the black box neural network can approximate every continuous function of $t,\q,\qd$ by error $\varepsilon>0$ under some norm $\Vert \cdot\Vert$ on function space $\mathcal{C}(\q,\qd,t)$, i.e., $\forall f\in \mathcal{C}(\q,\qd,t)$, there exists a $\boldsymbol{w}_{\mathrm{n}}^f$, such that, $\Vert f_{\mathrm{n}}^{NN}(\q,\qd,t,\boldsymbol{w}_{\mathrm{n}}^f)-f(\q,\qd,t) \Vert\le \varepsilon$. Furthermore, assume there exists $\boldsymbol{w}_0$, such that,  $f_{\mathrm{n}}^{NN}(\q,\qd,t;\boldsymbol{w}_0)\equiv 0$.
	Then, given any continuous function $f$, the following claim stands:
	\begin{itemize}
		\item For $\lambda>1$, 
		\begin{align*}
		&\arg\min_{\boldsymbol{w}_{\mathrm{c}},\boldsymbol{w}_{\mathrm{n}}} \Vert f(\q,\qd,t)-f_{\mathrm{c}}^{NN}(\q,\qd;\boldsymbol{w}_{\mathrm{c}})-f_{\mathrm{n}}^{NN}(\q,\qd,t;\boldsymbol{w}_{\mathrm{n}})\Vert +\lambda \Vert f_{\mathrm{n}}^{NN}(\q,\qd,t;\boldsymbol{w}_{\mathrm{n}})\Vert
			\\
		=&\left(\arg\min_{\boldsymbol{w}_{\mathrm{c}}} \Vert f(\q,\qd,t)-f_{\mathrm{c}}^{NN}(\q,\qd;\boldsymbol{w}_{\mathrm{c}})\Vert,0\right). 
		\end{align*}
		\item For $\lambda<1$, for any ($\boldsymbol{w}_{\mathrm{c}}^0$, $\boldsymbol{w}_{\mathrm{n}}^0$) $\in \arg\min_{\boldsymbol{w}_{\mathrm{c}},\boldsymbol{w}_{\mathrm{n}}} \Vert f(\q,\qd,t)-f_{\mathrm{c}}^{NN}(\q,\qd;\boldsymbol{w}_{\mathrm{c}})-f_{\mathrm{n}}^{NN}(\q,\qd,t;\boldsymbol{w}_{\mathrm{n}})\Vert +\lambda \Vert f_{\mathrm{n}}^{NN}(\q,\qd,t;\boldsymbol{w}_{\mathrm{n}})\Vert$, 
		\begin{gather}
			\label{eq:n_NN_approximating}
			\Vert f_{\mathrm{n}}^{NN}(\q,\qd,t;\boldsymbol{w}^0_{\mathrm{n}})-f(\q,\qd,t)+f_{\mathrm{c}}^{NN}(\q,\qd;\boldsymbol{w}^0_{\mathrm{c}})\Vert\le \frac{1+\lambda}{1-\lambda} \varepsilon,
			\\
			\label{eq:c_NN_approximating}
			\Vert f(\q,\qd,t)-f_{\mathrm{c}}^{NN}(\q,\qd;\boldsymbol{w}^0_{\mathrm{c}})\Vert \le \min_{\boldsymbol{w}_{\mathrm{c}}} \Vert f(\q,\qd,t)-f_{\mathrm{c}}^{NN}(\q,\qd;\boldsymbol{w}_{\mathrm{c}})\Vert+(1+\lambda)\varepsilon . 
		\end{gather}
	\end{itemize}
\end{corollary}
\begin{proof}
	When $\lambda>1$, the claim can be proved following exactly the same routine as Theorem \ref{thm:1}. When $\lambda<1$, we follow the same routine as Theorem \ref{thm:1} to decompose the optimization problem into a two step minimization problem: fixed $\boldsymbol{w}_{\mathrm{c}}$, let $g(\q,\qd,t)=f(\q,\qd,t)-f_{\mathrm{c}}^{NN}(\q,\qd;\boldsymbol{w}_{\mathrm{c}})$. Then, there exists a $\boldsymbol{w}_{\mathrm{n}}^g$, such that 
	\begin{equation*}
		\Vert f_{\mathrm{n}}^{NN}(\q,\qd,t;\boldsymbol{w}^g_n)-g(\q,\qd,t)\Vert\le \varepsilon.
	\end{equation*}
	Therefore, 
	\begin{align*}
		&\min_{\boldsymbol{w}_{\mathrm{n}}}\left( \Vert f(\q,\qd,t)-f_{\mathrm{c}}^{NN}(\q,\qd;\boldsymbol{w}_{\mathrm{c}})-f_{\mathrm{n}}^{NN}(\q,\qd,t;\boldsymbol{w}_{\mathrm{n}})\Vert +\lambda \Vert f_{\mathrm{n}}^{NN}(\q,\qd,t;\boldsymbol{w}_{\mathrm{n}})\Vert\right)
		\\
		\le&\Vert f(\q,\qd,t)-f_{\mathrm{c}}^{NN}(\q,\qd;\boldsymbol{w}_{\mathrm{c}})-f_{\mathrm{n}}^{NN}(\q,\qd,t;\boldsymbol{w}^g_n)\Vert+\lambda \Vert f_{\mathrm{n}}^{NN}(\q,\qd,t;\boldsymbol{w}^g_n)\Vert
		\\
		\le&(1+\lambda)\Vert f(\q,\qd,t)-f_{\mathrm{c}}^{NN}(\q,\qd;\boldsymbol{w}_{\mathrm{c}})-f_{\mathrm{n}}^{NN}(\q,\qd,t;\boldsymbol{w}^g_n)\Vert  +\lambda \Vert f(\q,\qd,t)-f_{\mathrm{c}}^{NN}(\q,\qd;\boldsymbol{w}_{\mathrm{c}})  )\Vert
		\\
		\le& \lambda \Vert f(\q,\qd,t)-f_{\mathrm{c}}^{NN}(\q,\qd;\boldsymbol{w}_{\mathrm{c}})  )\Vert+(1+\lambda)\varepsilon.
	\end{align*}
	Let $\boldsymbol{w}_{\mathrm{n}}=\arg\min_{\boldsymbol{w}_{\mathrm{n}}}\left( \Vert f(\q,\qd,t)-f_{\mathrm{c}}^{NN}(\q,\qd;\boldsymbol{w}_{\mathrm{c}})-f_{\mathrm{n}}^{NN}(\q,\qd,t;\boldsymbol{w}_{\mathrm{n}})\Vert +\lambda \Vert f_{\mathrm{n}}^{NN}(\q,\qd,t;\boldsymbol{w}_{\mathrm{n}})\Vert\right)$. Then,
	\begin{align}
		\nonumber
		&(1-\lambda)\Vert f(\q,\qd,t)-f_{\mathrm{c}}^{NN}(\q,\qd;\boldsymbol{w}_{\mathrm{c}})-f_{\mathrm{n}}^{NN}(\q,\qd,t;\boldsymbol{w}_{\mathrm{n}})\Vert
		\\
		\nonumber
		\le&\Vert f(\q,\qd,t)-f_{\mathrm{c}}^{NN}(\q,\qd;\boldsymbol{w}_{\mathrm{c}})-f_{\mathrm{n}}^{NN}(\q,\qd,t;\boldsymbol{w}_{\mathrm{n}})\Vert +\lambda \Vert f_{\mathrm{n}}^{NN}(\q,\qd,t;\boldsymbol{w}_{\mathrm{n}})\Vert
		\\
		\nonumber
		-&\lambda \Vert f(\q,\qd,t)-f_{\mathrm{c}}^{NN}(\q,\qd;\boldsymbol{w}_{\mathrm{c}})  )\Vert
		\\
		\label{eq:bound_for_n_NN}
		\le& (1+\lambda)\varepsilon,
	\end{align}
	which finishes the proof of Eq. (\ref{eq:n_NN_approximating}). 
	\\Let $h(\boldsymbol{w}_{\mathrm{c}})=\min_{\boldsymbol{w}_{\mathrm{n}}}\Vert f(\q,\qd,t)-f_{\mathrm{c}}^{NN}(\q,\qd;\boldsymbol{w}_{\mathrm{c}})-f_{\mathrm{n}}^{NN}(\q,\qd,t;\boldsymbol{w}_{\mathrm{n}})\Vert+\lambda \Vert f_{\mathrm{n}}^{NN}(\q,\qd,t;\boldsymbol{w}_{\mathrm{n}})\Vert$. By Eq. (\ref{eq:bound_for_n_NN}), 
	\begin{equation*}
		\Vert h(\boldsymbol{w}_{\mathrm{c}})\Vert -\lambda\Vert f(\q,\qd,t)-f_{\mathrm{c}}^{NN}(\q,\qd;\boldsymbol{w}_{\mathrm{c}})\Vert\le (1+\lambda)\varepsilon,
	\end{equation*}
	which further leads to 
	\begin{equation*}
		\min_{\boldsymbol{w}_{\mathrm{c}}}\Vert h(\boldsymbol{w}_{\mathrm{c}})\Vert \le (1+\lambda)\varepsilon+\lambda\min_{\boldsymbol{w}_{\mathrm{c}}}\Vert f(\q,\qd,t)-f_{\mathrm{c}}^{NN}(\q,\qd;\boldsymbol{w}_{\mathrm{c}})\Vert.
	\end{equation*}
	\\This completes the proof.
\end{proof}

\setcounter{secnumdepth}{2}

\end{document}